%% file: main.tex
\documentclass[sigplan,screen,acmthm]{acmart}

\AtBeginDocument{%
  \providecommand\BibTeX{{%
    \normalfont B\kern-0.5em{\scshape i\kern-0.25em b}\kern-0.8em\TeX}}}

\setcopyright{acmcopyright}
\copyrightyear{2018}
\acmYear{2018}
\acmDOI{10.1145/1122445.1122456}

\acmConference[Woodstock '18]{Woodstock '18: ACM Symposium on Neural
  Gaze Detection}{June 03--05, 2018}{Woodstock, NY}
\acmBooktitle{Woodstock '18: ACM Symposium on Neural Gaze Detection,
  June 03--05, 2018, Woodstock, NY}
\acmPrice{15.00}
\acmISBN{978-1-4503-XXXX-X/18/06}

\input{macros}

\copyrightyear{2022}
\acmYear{2022}
\setcopyright{acmcopyright}\acmConference[LICS '22]{37th Annual ACM/IEEE Symposium on Logic in Computer Science (LICS)}{August 2--5, 2022}{Haifa, Israel}
\acmBooktitle{37th Annual ACM/IEEE Symposium on Logic in Computer Science (LICS) (LICS '22), August 2--5, 2022, Haifa, Israel}
\acmPrice{15.00}
\acmDOI{10.1145/3531130.3533333}
\acmISBN{978-1-4503-9351-5/22/08}

\begin{document}

\title{Solvability of orbit-finite systems of linear equations}
%{Linear equations with atoms}


\author{Arka Ghosh}
\authornote{Partially supported by the  NCN grant 2019/35/B/ST6/02322.}
\orcid{0003-3839-8459}
\affiliation{%
 \institution{University of Warsaw}
 \country{Poland}
}
\author{Piotr Hofman}
\authornote{Partially supported by the NCN grant 2016/21/D/ST6/01368.}
\orcid{0001-9866-3723}
\affiliation{%
 \institution{University of Warsaw}
 \country{Poland}
}
\author{S{\l}awomir Lasota}
\authornote{Partially supported by the ERC Starting grant INFSYS, agreement no.~950398.}
\orcid{0001-8674-4470}
\affiliation{%
 \institution{University of Warsaw}
 \country{Poland}
}

% \author{Sgt. Pepper}
% \affiliation{\institution{Lonely Hearts Club Band}\country{Liverepool}}


%\renewcommand{\shortauthors}{A. Ghosh, P. Hofman,  S. Lasota}

\begin{abstract}
We study orbit-finite systems of linear equations, in the setting of sets with atoms.
Our principal contribution is a decision procedure for solvability of such systems.
The procedure works for every field (and even commutative ring) under mild effectiveness assumptions,
and reduces a given orbit-finite system to a number of finite ones:
exponentially many in general, but polynomially many when the atom dimension of input systems is fixed.
Towards obtaining the procedure we push further the theory of vector spaces generated by orbit-finite sets,
and show that each such vector space admits an orbit-finite basis.
This fundamental property is a key tool in our development, but should be also of  wider interest.
%
% vector spaces and systems of linear equations in the setting of sets with atoms.
%In this setting finiteness is relaxed to orbit-finiteness, i.e., finiteness up to permutation of atoms.
%Accordingly, we consider vector spaces generated by a possibly infinite orbit-finite set.
%As our first technical result we prove that every such orbit-finitely generated
%vector space has an orbit-finite basis. 
%This serves as a tool in proving our primary result: it is decidable if 
%a given orbit-finite system of linear equations has a solution.
\end{abstract}


%mandatory: Please choose ACM 2012 classifications from https://dl.acm.org/ccs/ccs_flat.cfm 

 %mandatory; please add comma-separated list of keywords

\maketitle

\input{intro}

\input{atoms}

\input{equations}

\input{bases-simplified}

\input{fin-solv-ph}

     %\input{fin-solv-sl}
     %\input{orbit_finite_sum}
\input{solv-2-fin-solv-simplified}

     %\input{integer_equations}
     %\input{equations-with-constr}
     %\input{petri_nets}

\input{finalrem}


%\begin{acks}
%\TODO{?}
%\end{acks}

\newpage
\bibliographystyle{ACM-Reference-Format}
\bibliography{bib}


\appendix
\input{app-missing-proofs}

\end{document}

%% file: macros.tex
\usepackage{todonotes}
\usepackage{algorithm}
\usepackage[noend]{algpseudocode}
\usepackage{xspace}
\usepackage{pgf}
\usepackage{tikz}
\usetikzlibrary{arrows,automata}
\usepackage{amsmath}
\usepackage{amsthm}
\usepackage{amsfonts}
\usepackage{comment}
\usepackage{mathtools}
\usepackage{xcolor}

\usepackage{amssymb}

\usepackage{thm-restate}
\usepackage{MnSymbol}

\newenvironment{slremark}
    {\begin{remark}\rm}
    {\hfill $\triangleleft$ \end{remark}}

\newenvironment{slexample}
    {\begin{example}}
    {\hfill $\triangleleft$ \end{example}}

\newenvironment{claimproof}
    {\begin{proof}} %[Proof of claim]}
    {\end{proof}}

\newtheorem{claim}{Claim}%[section]
\newtheorem{remark}{Remark}%[section]
\theoremstyle{definition}

\newcommand{\zerovec}{{\vr 0}}

\usepackage{amsmath}

\newcommand{\symg}[1]{\text{\sc S}_{#1}}
\newcommand{\faktor}[2]{#1/#2}
\newcommand{\A}{\text{\sc Atoms}}

\newcommand{\basisof}[1]{\widehat{#1}}

\newcommand{\Q}{\mathbb{K}}
\newcommand{\Rat}{\mathbb{Q}}

\newcommand{\Nat}{\mathbb{N}}
\newcommand{\Int}{\mathbb{Z}}

\newcommand{\supp}[1]{\text{sup}(#1)}
\newcommand{\dimm}[2]{#1\text{-dim}(#2)}

\newcommand{\dom}[1]{\text{dom}(#1)}
\newcommand{\odom}[1]{S\text{-orbit-dom}(#1)}

\newcommand{\toremove}[1]{#1} %{\textcolor{olive}{#1}}

\newcommand{\prob}[3]{
\begin{quote}
{\sc #1:}
\begin{description}
\item[Input:] #2.
\item[Question:] #3?
\end{description}
\end{quote}
}

\newcommand{\Lin}[1]{\text{\sc Fin-Lin}(#1)}
\newcommand{\GLin}[1]{\text{\sc Lin}(#1)}

\newcommand{\Span}[2]{\text{\sc Fin-Span}_{}(#2)}
\newcommand{\GSpan}[2]{\text{\sc Span}_{}(#2)}

\newcommand{\setof}[2]{\left\{\, #1 \; \middle| \; #2 \, \right\}}
\newcommand{\setofbeg}[2]{\{\, #1 \; | \; #2}
\newcommand{\setofend}[1]{\,\, #1 \, \}}
\newcommand{\set}[1]{\left\{ #1 \right\}}
\newcommand{\para}[1]{\subsubsection*{\bf #1.}}

\newcommand{\Aut}[1]{\text{\sc Aut}_{#1}}
\newcommand{\size}[1]{|#1|}

\newcommand{\spread}[1]{\widetilde{#1}}

\newcommand{\innerprod}[2]{#1 \cdot #2 }
\newcommand{\subseteqfin}{\subseteq_\text{fin}}

\newcommand{\vr}[1]{\mathbf{#1}}
\newcommand{\otu}[2]{#1^{(#2)}}
\newcommand{\utu}[2]{{#1 \choose #2}}
\newcommand{\constvr}[2]{{\vr {#1}}_{#2}}

\newcommand{\solvname}{{\sc Solv}}
\newcommand{\finsolvname}{{\sc Fin-Solv}}
\newcommand{\belt}{{\vr w}}

\newcommand{\mult}[2]{#1 \cdot #2}
\newcommand{\restr}[2]{{#1}{\restriction}_{#2}}
\newcommand{\quot}[1]{\pi_{#1}}

\newcommand{\mdf}[1]{\overline{#1}}

\newcommand{\gcog}[2]{[#2](#1)}    %\text{\sc cog}({#1,#2,#3})}

\newcommand{\tofs}{\to_\text{fs}}
\newcommand{\tofin}{\to_\text{fin}}
\newcommand{\main}[1]{\widetilde{#1}} %{\ddot{#1}} %\widetilde{#1}}
\newcommand{\comp}[2]{#1_#2}
\newcommand{\Aul}[2]{\A^{#1}_{#2}}
\newcommand{\Deltaw}[1]{\Delta^{\!\!\prec}#1}
\newcommand{\Deltawp}[2]{\Delta^{\!\!\prec_{#1}}#2}
\newcommand{\Ord}{\mathcal{O}}
\newcommand{\TO}[2]{\text{\sc TO}(#1,#2)}

\newcommand{\pow}[2]{\mathcal{P}_{#1}(#2)}
\newcommand{\at}[1]{#1} % {\hat{#1}}
\newcommand{\magicto}{\twoheadrightarrow} %{\dashedrightarrow}

\renewcommand{\a}{{\at \alpha}}
\renewcommand{\b}{{\at \beta}}
\newcommand{\g}{{\at \gamma}}
\renewcommand{\d}{{\at \delta}}
\renewcommand{\c}{\g}

\newcommand{\PTIME}{\textsc{P}\xspace}
\newcommand{\NP}{\textsc{NP}\xspace}

\newcommand{\ETIME}{\textsc{ExpTime}\xspace}
\newcommand{\ackermann}{\textsc{Ackermann}\xspace}

%% file: intro.tex
% !TEX root = main.tex

\section{Introduction}

Applications of linear algebra, and in particular of systems of linear equations,
are ubiquitous in computer science (see e.g.~\cite{Colcombet15,SilvaTC96,Cormenbook}). 
%\slawek{any other good references?}
%\arka{Maybe we don't need to give a reference here}
%\PH{linear programing, computational geometry, 
%computer graphics, I think we can add Cormen Introduction to algorithms, on }
In this paper, motivated by recent and potential future applications to analysis of data-enriched 
models~\cite{DBLP:conf/lics/HofmanLT17,HL18,DBLP:conf/fossacs/GuptaSAH19,BKM21}, 
we augment systems of linear equations with \emph{atoms}~\cite{atombook,Pitts:book}  (also called data values)
thus shifting from finite to orbit-finite systems.
The infinite sets that we study are constructed using atoms
which can only be accessed in a very limited way, namely can only be tested for equality.

Fix a countably infinite set $\A=\set{\at 1, \at 2, \at 3, \ldots}$, 
whose elements are called \emph{atoms}, assuming that the only 
operations on atoms are (dis)equality tests.
%
%\para{Example}
%\begin{slexample}
As an example, consider pairs of distinct atoms
$C=\setof{\a\b\in\A^2}{\a\neq\b}$
 as unknowns (for succinctness, here and in the sequel we write ordered pairs $(\a, \b)$ of atoms as $\a\b$, and likewise
 for triples), 
 and the infinite system of equations
\begin{align*}
\a\b - 2{\cdot} \b\g + \g\a \ = \ 1 \qquad (\a,\b, \g\in\A,
 \a\neq\b\neq \g\neq \a).
\end{align*}
The system is finitely described by the above formula using only (dis)equalities between atoms, 
and therefore is invariant under all permutations of atoms.
Furthermore, up to permutation of atoms the system consists of just one equation -- it is one \emph{orbit};
in the sequel we consider orbit-finite systems (finite unions of orbits).
Each unknown $\a\b\in C$ is determined (\emph{supported}) by 2 atoms (its \emph{atom dimension} is 2) 
while each equation by 3 atoms, therefore the atom dimension
of the whole example system is 3.
The example equations are finite, but need not to be so in general.
Our primary goal is to algorithmically test if such a system has a solution, that is
a rational assignment $\vr x : C\to \Rat$, or maybe an integer assignment $\vr x : C\to\Int$, that satisfies all the equations,
i.e.,
$$
\vr x(\a\b) - 2{\cdot}\vr x(\b\g) + \vr x(\g \a) \ = \ 1
$$
for every $\a\b\c\in \A^3$ such that $\a\neq\b\neq \g\neq \a$.

We use the language of linear algebra.
For instance, a solution is a vector over $C$ (belongs
to the vector space generated by $C$),
and the above system 
may be presented as an infinite matrix plus the infinite right-hand side vector:
\newcommand{\entd}{\!-2\ }
\newcommand{\entj}{\ 1\ }
\newcommand{\entz}{ \ 0\ }
\begin{align*}
& \qquad\ \ \ \,  
\textcolor{gray}{\begin{matrix}
\at 1 \at 2\, & \at 1 \at 3\, & \at 2 \at 3\, & \at 3 \at 4\, & \at 3 \at 1\, & \at 4 \at 1\, & \at 4 \at 2\, & \hdots
\end{matrix}}
\\
& 
\textcolor{gray}{
\begin{matrix}
\at 1 \at 2 \at 3 \\
\at 2 \at 3 \at 4 \\
\at 1 \at 3 \at 4 \\
\at 3 \at 1 \at 2 \\
\vdots
\end{matrix}}
\quad
\textcolor{black}{
\begin{bmatrix}
\ \entj & \entz & \entd & \entz & \entj & \entz & \entz & \hdots \ \  \\
\ \entz & \entz & \entj & \entd & \entz & \entz & \entj & \hdots \ \ \\
\ \entz & \entj & \entz & \entd & \entz & \entj & \entz & \hdots \ \ \\
\ \entd & \entz & \entj & \entz & \entj & \entz & \entz & \hdots \ \  \\
\ \vdots & \vdots & \vdots & \vdots & \vdots & \vdots & \vdots & \ddots \ \ 
\end{bmatrix}
}
\qquad
\textcolor{black}{
\begin{bmatrix}
\ 1 \ \\
\ 1 \ \\
\ 1 \ \\
\ 1 \ \\
\vdots
\end{bmatrix}
}
\end{align*}
The columns of the matrix are indexed by pairs $\a\b\in C$, and rows by triples $\a\b\c\in \A^3$ where $\a\neq\b\neq \g\neq \a$.
%\end{slexample}

\para{Contribution}
As the main contribution, we provide an algorithm for solvability of orbit-finite systems of linear equations.
More formally, our algorithm accepts as input a system consisting of an orbit-finite matrix $\vr A$
and a right-hand side vector $\vr t$,
both \emph{finitely-supported} (i.e., definable using finitely many fixed atoms, hence finitely presentable).
The algorithm checks whether the given system admits a solution which is also finitely-supported (hence
also finitely presentable). 

The coefficients in $\vr A$ and $\vr t$, as well as in the solutions, are assumed to come from an arbitrary fixed 
commutative ring $(\Q, 0, 1, +, \cdot)$ 
%\slawek{maybe $\R$ better than $\Q$?}
%\arka{or maybe $\mathbb{K}$? I think $\R$ has become the standard notation for the set of real numbers.}
which is assumed to be \emph{effective}:
%\arka{Missing countability assumption which is mentioned while explaining effectiveness in the beginning of section
%\ref{sec:spaces}.
%Although it is implied by finite representability.}
its elements are finitely representable; equality is decidable for these representations;
ring operations (addition, subtraction, multiplication) are computable using the representations;
and solvability of finite systems over $\Q$ is decidable.
%\slawek{I don't see how to drop any of these effectiveness assumptions}\PH{
%Let us do $a+b$. The ring is enumerable, so you enumerate $c$ and check equation if $a+b=c$. 
%If you want to have variables then you can write system of equations $a+b + x= c$ and $x=0$.
%There is a solution if and only if $c=a+b$. The same way you can do $-$.
%To do multiplication $a*b$ you do $x=b$ and $y=0$ and $a*x + y = c$ and you enumerate $c$ and test if there is a solution.
%}
Examples abound: the rational field $\Rat$; the integer ring $\Int$; finite commutative rings;
the field of algebraic numbers; the field of complex numbers.

In brief, the algorithm computes a number of finite systems of linear equations over $\Q$
and answers positively exactly when all these systems are solvable.
The number of finite systems and their sizes are exponential in general;
however, once the atom dimension of the input system is fixed, the algorithm computes only polynomially many
finite systems of polynomial size.
In particular, for fixed atom dimension we obtain polynomial time procedures
for solvability over $\Rat$ or $\Int$. 
On the way we also provide an algorithm for \emph{finitary} solvability where one only seeks solutions which
assign zero to almost all unknowns.

On the mathematical level, we
push further the theory of orbit-finitely generated vector spaces initiated in~\cite{BKM21},
in order to obtain a key tool for our algorithmic considerations:
we show that each orbit-finitely generated vector space admits an orbit-finite basis.
We believe that this finding is of independent wider interest.

\para{Outline}
After preliminaries on sets with atoms, % in Section~\ref{sec:atoms}, 
in Section~\ref{sec:spaces}
we introduce orbit-finitely generated vector spaces and  Orbit-finite Basis Theorem, and
in Section~\ref{sec: linear eq} 
we introduce orbit-finite systems of linear equations and formulate the main result. % (Theorem~\ref{thm:solv}). 
The remaining sections contain the proofs.
Some missing parts thereof are to be found in the full version~\cite{fullver}.
%\arka{Citation not working}

\para{Motivations}
The main motivation for this work comes from past and potential future applications in analysis of computation models 
enriched with data, including different kinds of automata over infinite alphabets~\cite{NSV04,KF94,lmcs14}. 
For example, while studying Parikh images~\cite{HJLP21} of register automata~\cite{KF94} or 
register context-free grammars~\cite{lmcs14,atombook,CL15-csl}, one works with nonnegative integer vectors of the form
$\Sigma\to\Nat$, where $\Sigma$ is an infinite alphabet.
%A key property of these models is that runs are closed under application of atom permutations, 
%and hence so are their Parikh images and, consequently, potentially definable only by infinite systems of equations.
Another potential application of orbit-finite systems of linear equations 
is the recently proposed algorithm for equivalence of 
weighted register automata, including unambiguous register automata~\cite{BKM21}.

Numerous applications arise in
data-enriched Petri nets \cite{Nets-with-Tokens-which-Carry-Data,Lasota16} 
(or vector addition systems~\cite{Coverability-Trees-for-Petri-Nets-with-Unordered-Data}), 
an extension of classical Petri nets~\cite{SilvaTC96} 
where tokens carry atoms (data values) that are compared by transitions.
In case when tokens are restricted to carry single data values (atom dimension 1)
one obtains a well structured transition system and
hence standard decision problems like coverability or boundedness are 
decidable~\cite{Nets-with-Tokens-which-Carry-Data,Coverability-Trees-for-Petri-Nets-with-Unordered-Data,What-Makes-Petri-Nets-Harder-to-Verify-Stack-or-Data,Lasota16}. % but Ackermann~hard.
Status of the reachability problem is unknown;
since integer linear equations form a crucial component in a decision procedure for reachability of classical Petri
nets~\cite{Mayr81,Kosaraju82,Lambert92,LerouxS15}, 
lifting the procedure to data-enriched setting would require solving orbit-finite systems of integer linear equations.
In case when tokens may carry tuples of atoms (arbitrary atom dimension) all the standard problems are undecidable~\cite{Lasota16}. 
Decidability may be regained by resorting to relaxations: % of data-enriched Petri nets: 
continuous semantics~\cite{DBLP:conf/fossacs/GuptaSAH19} allowing for fractional executions of transitions, 
or so-called integer semantics~\cite{DBLP:conf/lics/HofmanLT17} dropping non-negativeness restriction on configurations.
Both these results have been obtained by reduction to solving certain systems of linear equations. 
%
%Worth mentioning is also a work on constraint satisfaction problem~\cite{DBLP:conf/lics/KlinKOT15} on infinite 
%(but orbit finite) 
%structures. 

\para{State of the art}
Our results generalise, or are closely related to, some earlier % (in terms of decidability but not complexity) 
partial results~\cite{DBLP:conf/lics/HofmanLT17,DBLP:conf/lics/KlinKOT15,HR21}.
%\arka{Should we add reference to paper by BKM, since they proved $\GLin{B}$ has orbit-finite spanning set} 

Systems of linear equations in~\cite{DBLP:conf/lics/HofmanLT17} have row indexes of atom dimension 1 
in which case finitary solvability is in \PTIME over $\Int$ or $\Rat$, and in \NP over $\Nat$.
In a more general but still restricted case studied in~\cite{HR21}, where in particular 
all row indexes are assumed to have the same atom dimension, 
finitary solvability is still in \PTIME
%\PH{for a fixed atom dimension only, if the atom dimension (more or less, the true parameter is the atom dimension of a 
%row - size of the support of the matrix)  is not fixed then it gets to exptime} 
over $\Int$ or $\Rat$, but in \ETIME over $\Nat$, both for fixed atom dimension.
Columns of a matrix are assumed to be finitary in~\cite{DBLP:conf/lics/HofmanLT17,HR21}.
Systems in another related work~\cite{DBLP:conf/lics/KlinKOT15} are over a finite field, contain only 
finite equations, and are studied as a special case of orbit-finite constraint satisfaction problems;
furthermore, solutions sought are not restricted to be finitely-supported.

Additionally, the work~\cite{HL18} investigates system of linear equations, in atom dimension 1,
 over ordered atoms:
solvability is in \PTIME over $\Int$ or $\Rat$, but equivalent to VAS reachability  
(and hence \ackermann-complete~\cite{CO,L,Lasota22}) over $\Nat$.

Our Orbit-Finite Basis Theorem is a follow-up and strengthening of Theorem VI.4 in~\cite{BKM21}:
each orbit-finitely generated vector space has an orbit-finite \emph{spanning} set.
  
%in which case 
%authors consider data vectors of a form ($l$-element sets of atoms) $\tofs \Int^k$, and prove that for a fixed $l$ the system of equations can 
%be solved in polynomial time over $\Int$ and in $\ETIME$ over $\Nat$.

%% file: atoms.tex
% !TEX root = main.tex

\section{Preliminaries on sets with atoms} \label{sec:atoms}

%\para{Sets with atoms} 
%
%\arka{I think there are several places where the mentioned fact is true for specific
%class of atoms including equality atoms, but not true for atoms in general.
%But I don't see anywhere where it's mentioned that we are working with equality atoms}
Our definitions rely on basic notions and results of the theory of \emph{sets with atoms}~\cite{atombook}, 
also known as nominal sets~\cite{Pitts:book}. 
%This paper is a 
%%In this section we recall what is necessary to follow our arguments; this is 
%part of the theory
%of computing in the realm of orbit-finite sets, developed e.g.~in~\cite{lics11,lmcs14,atombook}.
We only work with \emph{equality atoms} which have no additional structure except for
the equality.

We fix a countably infinite set $\A=\set{\at 1, \at 2, \at 3, \ldots}$, whose elements we call \emph{atoms}.
We reserve Greek letters $\a,\b,\g,  \ldots$ to range over atoms.
Informally speaking, a set with atoms is a set that can have atoms, or other sets with atoms,
as elements.
Formally, we define the universe of sets with atoms by a suitably adapted cumulative hierarchy of sets,
by transfinite induction: % on ordinals $\gamma$:
the only set of \emph{rank} 0 is the empty set; and for a cardinal $i$, a set of rank $i$ may contain, as elements, 
sets of rank smaller than $i$ as well as atoms.
%\arka{Should we use some other letters since we reserved greek letters for atoms?}
In particular, nonempty subsets $X\subseteq\A$ have rank 1.

The group $\Aut{}$ of all permutations of $\A$, called in this paper \emph{atom automorphisms},
acts on sets with atoms by consistently renaming all atoms in a given set. 
Formally, by another transfinite induction, for $\pi\in\Aut{}$ we define $\pi(X) = \setof{\pi(x)}{x\in X}$.
Via standard set-theoretic encodings of pairs or finite sequences we obtain, in particular, 
the pointwise action on pairs $\pi (x,y)=(\pi(x),\pi(y))$,
and likewise on finite sequences.  %: $\pi\cdot (x_1, \ldots, x_n)=(\pi\cdot x_1, \ldots, \pi\cdot x_n)$ .
Relations and functions from $X$ to $Y$ are considered as subsets of $X\times Y$.
%for instance, in case of $f:\A\to\A$, we have $\pi(f)(\a)=\pi(f(\pi^{-1}(\a)))$, or
%$\pi(f)(\pi(\a))=\pi(f(\a))$.
%If a set $X$ is \emph{pure}, i.e., no atom belongs to $X$, its elements, elements of its elements, etc., 
%then $\pi(X) = X$ for every $\pi\in \Gr$; for instance $\pi(n) = n$ for $n\in\Nat$.

We restrict to sets with atoms that only depend on finitely many atoms, in the following sense. 
For $S\subseteq \A$, let $\Aut S = \setof{\pi\in\Aut{}}{\pi(\a) = \a \text{ for every } \a \in S}$ be the set of atom automorphisms
that \emph{fix} $S$.
We call elements of $\Aut{S}$ \emph{$S$-atom automorphisms}.
A \emph{support} of $x$ is any finite set $S\subseteqfin\A$ (we use the symbol $\subseteqfin$ for finite subsets)
such that  for all
$\pi\in\Aut S$ it holds $\pi(x)=x$.
%\end{quote}
In this case we also say that $x$ is \emph{$S$-supported}.
As a special case, a function $f$ is supported by $S$ if 
$f(\pi(x)) = \pi(f(x))$ for every argument $x$ and $\pi\in\Aut{S}$.
An $S$-supported set is also $S'$-supported, as long as $S\subseteq S'$.
An element (or set) $x$ is \emph{finitely supported} if it has some finite support; in this case
%Finite supports of a fixed element are closed under intersections, and hence every finitely supported 
$x$ has {\em the least support}, denoted $\supp x$, called \emph{the support} of $x$ 
(cf.~\cite[Sect.~6]{atombook}).
%(c.f.~\cite[Prop.~2.3]{Pitts:book}, \cite[Cor.~9.4]{lmcs14}).
Sets supported by $\emptyset$ we call \emph{equivariant}.

For instance, given $\a,\b\in\A$, the support of the set $\A\setminus\set{\a,\b}$ is $\set{\a,\b}$;
in general, a set is $S$-supported if and only if it is invariant under all $S$-atom automorphisms.
The set $\A^2$ and the projection function $\pi_1 : \A^2 \to \A : (\a, \b) \mapsto \a$
are both equivariant; 
% (indeed, every $\pi \in \Gr$ preserves the graph $\setof{\pair{\pair a b} a}{a,b\in\A}$ of $\pi_1$);
and the support of a tuple $(\a_1, \ldots, \a_n)\in\A^n$, encoded as a set in a standard way, is
the set of atoms $\set{\a_1, \ldots, \a_n}$ appearing in it.

From now on, we shall only consider sets  that are hereditarily finitely
supported, %(called briefly \emph{legal}), 
i.e., ones that have a finite support, whose every element has some finite support,
and so on.
%A set $X$ in the cumulative hierarchy is called \emph{legal} (or \emph{set with atoms})
%if it is hereditarily finitely-supported: it is finitely-supported itself, all its elements are finitely supported, etc.
%From now on all sets are silently assumed to be legal.

\vspace{-0.8mm}

\para{Orbit-finite sets} 

Let $S\subseteqfin\A$. Two atoms or sets with atoms $x, y$ % of an $S$-supported set  
are \emph{in the same $S$-orbit} if $\pi(x) = y$ for some $\pi\in\Aut S$.
This equivalence relation splits all atoms and sets with atoms
%every $S$-supported set $X$ 
into equivalence classes, which we call \emph{$S$-orbits};
$\emptyset$-orbits we call equivariant orbits.
%\slawek{these 2 sentences added:}
By the very definition, every $S$-orbit $O$ is $S$-supported: $\supp O \subseteq S$ and,
even if the inclusion is strict (which may happen only for singleton orbits), $O$ is also a $\supp O$-orbit.
%but the equality does not have to hold.
When the set $S$ is irrelevant, we simply speak of an \emph{orbit}, meaning an $S$-orbit for some $S\subseteqfin \A$.

Every $S$-supported set is a union of (necessarily disjoint) $S$-orbits; the set is
\emph{orbit-finite} if this union is finite.
Orbit-finiteness is stable under orbit-refinement: if $S\subseteq S'$, a finite union of $S$-orbits is also a finite union
of $S'$-orbits (but the number of orbits may increase). 
Examples of orbit-finite sets are: $\A$ (1 orbit); 
$\A - \set{\a}$ for some $\a\in\A$ (1 orbit);
$\A^2$ (2 orbits: diagonal and non-diagonal);
$\A^3$ (5 orbits, corresponding to equality types of triples);
%atom permutations fix elements of $H$);
non-repeating $n$-tuples of atoms (1 orbit)
\[
\otu \A n = \setof{\!(\a_1,\ldots,\a_n)\in\A^n\!}{\a_i\neq\a_j \text{ for all }i\neq j}\!;
\]
$n$-sets of atoms $\utu \A n = \setof{X\subseteq \A}{\size X = n}$ (1 orbit).

The set  $\pow {\text{fin}} \A$ of all finite subsets of atoms is orbit-infinite as cardinality is an invariant of each orbit.

\para{Orbit representation}

For a positive integer $k>0$, denote by $\symg k$ the symmetric group on $\set{1, \ldots, k}$.
Given a subgroup $G\leq \symg k$ of the symmetric group $\symg k$, 
we denote by $\faktor{\otu \A k}{G}$ 
the set of non-repeating $k$-tuples of atom modulo coordinate permutations from the group $G$.
More formally, we define an equivalence in $\otu \A k$, where a tuple $a = (\a_1, \ldots, \a_k)\in \otu \A k$ is equivalent
to every tuple $a \circ \sigma := (\a_{\sigma(1)}, \ldots, \a_{\sigma(k)})$, where $\sigma \in G$.
The equivalence classes are thus finite.
Then we define a canonical quotient 
$\quot G : \otu \A k \to \faktor{\otu \A k}{G}$ 
mapping a tuple $a\in \otu \A k$ to its equivalence class.

\begin{slexample}
Let $k = 3$ and $G \leq \symg 3$ be generated by the cyclic shift $\sigma$ to the right:
$\sigma \ = \ $
{\tiny $
\begin{pmatrix}
 1 & 2 & 3 \\
 2 & 3 & 1
 \end{pmatrix}
$}.
The quotient $\quot G : \otu \A 3 \to \faktor{\otu \A 3}{G}$ maps each triple $(\a,\b,\g)$ to 
$\set{(\a,\b,\g), (\g,\a,\b), (\b, \g, \a)}$.
\end{slexample}

\begin{lemma}[\cite{atombook}, Thm.~6.3] \label{lem:repr}
Every equivariant orbit is in equivariant bijection with
$\faktor{\A^{(k)}}{G}$ for some $k \in \Nat$ and some subgroup $G\leq \symg k$.
\end{lemma}
%
%In the special case when the group $G$ is trivial, the equivariant orbit is in equivariant bijection with $\otu \A k$;
%such orbits we call \emph{straight} orbits.
%

%% file: equations.tex
% !TEX root = main.tex

\section{Orbit-Finite Basis Theorem}\label{sec:spaces}

%For the sake of concreteness we work in this paper with vector spaces over the rational field $\Q$.

\para{Proviso}
Throughout the paper we fix a countable commutative ring $(\Q, 0, 1, +, \cdot)$ with multiplicative unit $1$,
and assume that the ring is \emph{effective}:
its elements are finitely representable and solvability of finite systems of linear equations is decidable.
As a direct consequence, equality is decidable for the element representations,
and the ring operations (addition, subtraction, multiplication) are  
computable using the representations.
The most prominent examples are rationals $\Rat$ and integers $\Int$.
%\arka{Should we add algebraic numbers and finite fields to the examples?}
%\slawek{new setting!}\PH{I think that finite representation + solvability of equations is enough}

%Furthermore, since we never use division,
%the results are still valid when a field is replaced by a commutative ring,
%e.g.~the integer ring $\Int$, as long as the above effectiveness assumptions are satisfied.
%\slawek{is this remark fine?}
%\slawek{since we can use a ring, do we need to keep the subgroup $\X$ as a parameter?}
%\arka{I don't see why not}\PH{the problem is that in our assumptions we have effective inverse, which is false for rings. It has to be changed to we can solve 
%equations over a ring.} 

\para{Vectors}

We are investigating vector spaces%
\footnote{
Formally, in case when $\Q$ is not a field, we should use the term \emph{module}.
Since modules/vector spaces studied in this paper are of particularly simple kind,
we prefer to stick to a widely known term \emph{vector space}.
} 
generated by an orbit-finite set.
Let $B$ be a fixed orbit-finite set.
\begin{definition}
By a \emph{vector} over $B$ we mean any finitely-supported function $\vr v$ from $B$ to $\Q$, 
written $\vr v : B\tofs\Q$ (vectors are written using boldface).
\end{definition}
The set of all vectors over $B$ we denote by $\GLin B = B \tofs \Q$. 
It is a vector space, 
with pointwise addition and scalar multiplication:
for $\vr v, \vr v'\in\GLin B$, $b\in B$ and $q\in \Q$, we have
$(\vr v+ \vr v')(b) = \vr v(b) + \vr v'(b)$ and $(q \cdot \vr v)(b) = q\cdot \vr v(b)$.
%The generating set $B$ can be considered as \emph{the dimension} of $\GLin B$.
%\TODO{do we want to call B dimension?}
The space $\GLin B$ may be considered as  the vector space \emph{generated} by $B$,
and $B$ as its \emph{dimension}%
\footnote{
Not to be confused with \emph{atom dimension} introduced in Section~\ref{sec:bases}.
}.
We define the \emph{domain} of a vector $\vr v\in\GLin B$ as $\dom {\vr v} = \setof{b\in B}{\vr v(b)\neq 0}$.
%
%\para{Finitary vectors}
%
A vector $\vr v$ over $B$ is \emph{finitary}, written $\vr v : B\tofin \Q$, 
if $\vr v(b) = 0$ for all except finitely many $b\in B$
(i.e., $\dom {\vr v}$ is finite).
A finitary vector $\vr v$ with domain $\dom {\vr v} = \set{b_1, \ldots, b_k}$ such that
$\vr v(b_1) = q_1, \ldots, \vr v(b_k) = q_k$, 
may be identified with a formal linear combination of elements 
of $B$:
\begin{align} \label{eq:formalsum}
\vr v \ = \ q_1 \cdot b_1 + \ldots + q_k \cdot b_k.  %, \qquad  q_1,\ldots,q_k\in\Q\setminus\set{0}.
\end{align}
The subspace of $\GLin B$ consisting of all finitary vectors we denote by $\Lin B = B\tofin\Q$.
%The space $\GLin B$ can be considered as the dual space of $\Lin B$\PH{Why?, is it a definition or a lemma?}, since we
%restrict our considerations to finitely-supported vectors only.
For finite $B$ of size $\size B = n$, $\GLin B = \Lin B$ is isomorphic to $\Q^{n}$.
%
%\toremove{
%\begin{slexample}
%Let $B = \otu \A 2$.
%The domain $\dom v$ of a finitary vector $v\in \Lin B$
%can be seen as edges of a directed graph without self-loops whose nodes are atoms, and the vector
%$v$ is just a $\Q$-labelling of edges of this graph.
%%
%%\TODO{finish}
%%\PH{What do we want to present in this example?}
%\end{slexample}
%}
%\para{Characteristic vectors}

For a subset $X\subseteq B$, we denote by $\constvr 1 X\in\GLin B$ the characteristic function of $X$,
i.e., the vector that maps each element of  $X$ to $1$ and all elements of $B\setminus X$ to $0$:
\[
\constvr 1 X : b \mapsto \begin{cases}
1 & \text{ if } b \in X \\
0 & \text{ otherwise.}
\end{cases} 
\]
We write $\constvr 1 b$ instead of $\constvr 1 {\set{b}}$, and $\vr 1$ instead of $\constvr 1 B$.
We sometimes want to treat $B$ itself as a subset of $\Lin B$, identifying every $b\in B$ with
the vector $\constvr 1 b$, or equivalently with
the trivial linear combination $1\cdot b$ as in~\eqref{eq:formalsum}.
%Let $\emb B = \setof{\constvr 1 b}{b\in B} \subseteq \Lin B$.
%Note that the mapping
%$
%b  \mapsto  \constvr 1 b : B \to \emb B
%$
%allows us to treat elements of $B$ as vectors in $\Lin B$.
%

%\para{Finitary spans}

\begin{lemma} \label{lem:1O}
Consider $S\subseteqfin \A$ and an $S$-supported $\vr v \in \GLin {B}$. Then
\begin{enumerate}
\item[(i)] $\vr v$ is constant, restricted to every $S$-orbit $O\subseteq B$;
\item[(ii)] $\vr v$ is a linear combination of characteristic vectors $\constvr 1 O$ of $S$-orbits $O\subseteq B$.
\end{enumerate}
\end{lemma}
\begin{proof}
%
%Consider any vector $\vr v \in \GLin {B}$, and let $S = \supp{\vr v}$.
The first part follows immediately as $S$ supports ${\vr v}$.
This allows us to write
$\vr v(O) \in \Q$ in place of $\vr v(x)$ for  $x \in O$. 
As required in the second part, we have:
\begin{align} \label{eq:1O}
\vr v = \sum_O \vr v(O) \cdot \constvr 1 O,
\end{align}
where $O$ ranges over finitely many %pairwise-disjoint 
$S$-orbits $O\subseteq B$. 
%Therefore $\vr v\in  \Span {} {\setof{\constvr 1 O}{O\in\orbits B}}$, as required.
\end{proof}
%
%\begin{corollary} \label{cor:1O}
%$\GLin B = \Span {} {\setof{\constvr 1 O}{O\subseteq B \text{ an orbit 
%\PH{This notation look strange, maybe write the sentence $O$ are a finitely supported 
%orbits of $B$}}}}$.
%\end{corollary}

\para{Orbit-finite bases}
The set $\setof{\constvr 1 b}{b\in B}$ is, by the very definition, a basis of $\Lin B$.
As our first result we prove that whenever $B$ is orbit-finite, this set can be extended to
an orbit-finite basis of the larger space $\GLin B$:
\begin{theorem}[Orbit-Finite Basis Theorem]\label{thm:basis}
For every orbit-finite set $B$, the space $\GLin B$ has an orbit-finite basis.% $\basisof B\subseteq \GLin B$.
\end{theorem}
The result constitutes a useful tool in our subsequent considerations of solvability of systems of linear equations.
The proof is delegated to Section~\ref{sec:bases}.

\begin{slremark} \label{rem:effect}
Theorem~\ref{thm:basis}, as well as our subsequent results, are all effective.
Indeed,
the transformation from $B$ to $\basisof B$ is equivariant,
and the set $\basisof B$ as well as  
the transformation from $\vr v\in \GLin B$ to its basis representation in $\Lin{\basisof B}$ are
supported by $\supp B$, 
%\arka{$v \mapsto$ ``basis representation of $v$'' is supported by $\supp{B}$?}
and therefore all are subject to the general rule of thumb: 
(hereditarily) orbit-finite sets are finitely representable, and
all finitely-supported transformations between these sets are effectively computable
(for a detailed presentation we refer to~\cite{program-atoms} or~\cite[Sect.~4,8,9]{atombook}).
\end{slremark}

%Moreover, as a corollary of the proof we deduce that
%an orbit-finite basis $\basisof B$ of $\GLin B$ can be computed effectively from $B$ and, given a vector $v\in\GLin B$, 
%one can compute effectively its presentation in the base as an element of $\Lin {\basisof B}$:
%%
%\begin{corollary} \label{cor:basis}
%For every orbit-finite set $B$, one can compute an orbit-finite set $\basisof B$ such that the vector spaces 
%$\GLin B$ and $\Lin {\basisof B}$ are related by an equivariant linear bijection, which is also 
%computable.
%\TODO{work out a precise formulation}\PH{What is not precise enough here?}
%\arka{This is a corollary of the proof, not a direct corollary of the theorem.
%Probably it is better to mention it here to avoid confusion?}
%
%\arka{The precise formulation is I think that for each orbit-finite set $B$ the maps
%$B \mapsto \basisof{B}$ and $B \mapsto f_B$ are equivariant,
%where $f_B : \GLin{B} \to \Lin{\basisof{B}}$ is the equivariant linear isomorphism.}
%\end{corollary}

%\begin{slremark}
%Theorem~\ref{thm:basis} is valid for any field in place of $\Q$.
%\end{slremark}
%\arka{Given our new setting, do we need the above remark anymore?}

\begin{slexample}\label{ex:basis}
Let $B = \otu \A 2$. 
%Every $(\a,\b)\in B$ can be identified with $\constvr 1 {(\a,\b)} \in \Lin B$, i.e.,
%with a single-edge two-node directed graph labeled by $1$.
For $\g\in\A$, let $\g \_ = \setof{\g \a}{\a\in \A\setminus\set{\g}}\subseteq B$; and symmetrically
let $\_ \g = \setof{\a \g}{\a\in \A\setminus\set{\g}}\subseteq B$.
%The set $B'=\setof{\constvr 1 {(\a,\b)}}{(\a,\b)\in B}$ %, considered in this way as a subset of $\Lin B$, 
%is a basis of $\Lin B$.
One obtains a basis $\basisof B\subseteq\GLin B$ by extending $\setof{\constvr 1 {\a\b}}{\a\b\in B}$
with the constant vector $\constvr 1 {}$ 
that maps every pair $\a \b\in B$ to $1$, and also,
for every $\g\in\A$, with the characteristic vector $\constvr 1 {\g \_}$that maps all pairs in $\g \_$ to $1$ and all others to $0$, 
and the characteristic vector $\constvr 1 {\_ \g}$ that maps all pairs in $\_ \g$ to $1$ and all others to $0$.
%\begin{align*}
%\constvr 1 {} : (\a, \b) \mapsto 1
%\qquad &
%\constvr 1 {(\g, \_)} : (\a, \b) \mapsto \begin{cases}
%1 & \text{ if } \a = \g \\
%0 & \text{ otherwise}
%\end{cases}
%\\ %\quad
%& \constvr 1 {(\_, \g)}  : (\a, \b) \mapsto \begin{cases}
%1 & \text{ if } \b = \g \\
%0 & \text{ otherwise.}
%\end{cases}
%\end{align*}

Towards seeing that this is indeed a base,
consider any vector $\vr v \in \GLin {\otu \A 2}$.
Let $S = \supp{\vr v}$.
Let $O_{\bullet\bullet} = \otu {(\A \setminus S)} 2$; for $\a\in S$, let
$O_{\a \bullet} = \set{\a}\times(\A\setminus S)$ and
$O_{\bullet \a} = (\A\setminus S)\times\set{\a}$.
% and $q_1 = \vr v(O_1)$. 
Note that all these are $S$-orbits. 
%For $\a \in S$, let $q_{\a, \_} = \vr v(O)$ for $O = \set{\a}\times(\A\setminus S)$ and 
%let $q_{\_, \a} = \vr v(O)$ for $O = (\A\setminus S)\times\set{\a}$.
%Finally, for $\a, \b\in S$ let $q_{\a,\b} = \vr v(\a, \b)$.
The decomposition~\eqref{eq:1O} of $\vr v$ may be rewritten into: % the following inclusion-exclusion term:
\begin{align*} % \label{eq:v1O}
\vr v \quad = \quad  \vr v(O_{\bullet\bullet}) \cdot \constvr 1 {} \ \ +  \!\!
 \sum_{\quad\a\in S\quad} (\vr v(O_{\a\bullet}) - \vr v(O_{\bullet\bullet})) \cdot \constvr 1 {\a \_} & \ \ +  \\
 \sum_{\quad\b\in S\quad} (\vr v(O_{\bullet \b}) - \vr v(O_{\bullet\bullet})) \cdot \constvr 1 {\_ \b} & \ \ +  \\ 
 \sum_{\a \b\in \otu S 2} (\vr v(\a \b)  - \vr v(O_{\a\bullet}) - \vr v(O_{\bullet\b}) + \vr v(O_{\bullet\bullet})) \cdot \constvr 1 {\a \b}&. 
\end{align*}
%
%\begin{align*} % \label{eq:v1O}
%\vr v \quad = \quad  q_1 \cdot \constvr 1 {} \quad +\quad
% \sum_{\quad\a\in S\quad} (q_{\a, \_} - q_1) \cdot \constvr 1 {(\a, \_)} & \quad +  \\
% \sum_{\quad\b\in S\quad} (q_{\_, \b} - q_1) \cdot \constvr 1 {(\_, \b)} & \quad +  \\ 
% \sum_{(\a, \b)\in \otu S 2} (q_{\a, \b} - q_{\a,\_} - q_{\_,\b} + q_1) \cdot \constvr 1 {(\a, \b)}&. 
%\end{align*}
%
This yields a representation of $\vr v$ in the base $\basisof B$, and the representation is unique.
\end{slexample}

\section{Solving linear equations}\label{sec: linear eq}

%\para{Inner product}

We note that the inner product of two vectors $\vr x, \vr y\in\GLin B$, defined as
\[
\innerprod {\vr x} {\vr y} \ = \ \sum_{b\in B} \vr x(b)\, \vr y(b),
\]
is not always well-defined.
We consider the right-hand side sum as well-defined when there are only finitely many $b\in B$
for which both $\vr x(b)$ and $\vr y(b)$ are non-zero
%, i.e., finitely many $b \in B$ such that $\vr x(b)$ and $\vr y(b)$ are both non-zero
(equivalently, the intersection $\dom {\vr x} \cap \dom {\vr y}$ is finite).
In particular, the inner product $\innerprod {\vr x} {\vr y}$ is always well-defined when one of $\vr x, \vr y$ is finitary.
\begin{slremark}
Consider $\Q = \Rat$. % or $\Q = \Int$
%(the remark remains valid for some other rings $\Q$ without nonzero zero-divisors). % \slawek{added}
%\arka{I think these rings are called \emph{integral domains}. Probably we should mention it in the abstract as well}
%\slawek{no, as this assumption matters only in this remark}
Since vectors are finitely supported and hence (c.f.~Lemma \ref{lem:1O}) contain only finitely many different numbers, 
$\dom {\vr x} \cap \dom {\vr y}$ is finite exactly when 
%\slawek{should we prove this?}\PH{I think it is trivial, we don't need to}
%\arka{I don't think that is necessary}
the right-hand side sum is \emph{unconditionally convergent},
i.e., convergent to the same value irrespectively of the order in which
the elements $b\in B$ are enumerated%
\footnote{We are grateful to Szymon Toru{\'n}czyk for attracting our attention to unconditional convergence.}.
%\arka{\emph{convergence} is not an well defined concept for rings in general}
%\slawek{right - I've adapted the remark accordingly}
\end{slremark}
%\PH{Agree with Arka, I would formulate this remark only for $\Q$.}

\para{Systems of linear equations}

Fix an orbit-finite set $C$ (one can think of $C$ as an indexing set of columns of a matrix).
By a linear equation over $C$ we mean a pair $e = (\vr a, t)$ where 
$\vr a \in\GLin C$ is a vector of left-hand side coefficients and $t\in\Q$ is a right-hand side target value.
A solution of $e$ is any vector $\vr x\in\GLin C$ such that the inner product
$\innerprod {\vr a} {\vr x}$ is well-defined and equals $t$.
We may consider constrained solutions, e.g., 
%\emph{integer} ones (where $\vr x(c)\in\Int$ for all $c\in C$),
%\emph{nonnegative} ones (where $\vr x(c)\geq 0$ for all $c\in C$), or 
\emph{finitary} ones.

A system of linear equations is just an indexed set of equations over the same set $C$.
Formally, an orbit-finite system of linear equations (over $C$)
is any finitely-supported function $B\tofs \GLin C \times \Q$
from some orbit-finite indexing set $B$ (one may think of $B$ as an indexing set of rows of a matrix).
By projecting to the first component we get a function $\vr A : B\tofs \GLin C$
% i.e.~an orbit-finite subset of $\GLin C$ indexed by $R$,
which we call the \emph{matrix} of the system;
by projecting to the second component (the target) 
we get a finitely-supported function $\vr t : B \tofs \Q$, i.e., a vector in $\GLin B$, 
which we call  the \emph{target} of the system.
The representation $B\tofs \GLin C$ of the matrix may be equivalently written as
a finitely-supported function $\vr A : B \times C \tofs \Q$
(thus $\vr A \in \GLin {B\times C}$ and hence it deserves boldface).

Systems of linear equations, when input to algorithms,
are assumed in the sequel to be given by a matrix-target pair $(\vr A, \vr t)$:
\vspace{-4mm}
\begin{align*}
& \quad\ \ \ \,  \begin{matrix} \ \cdots \quad &  c  & \quad \cdots\ \end{matrix} \\
& \begin{matrix} \vdots \\ b \\ \vdots \end{matrix} \ \ 
\begin{bmatrix}
\  & \vdots & \ \\
\ \cdots & \vr A(b,c) & \cdots\ \  \\
      & \vdots & \  
\end{bmatrix}
\qquad
\begin{bmatrix}
 \vdots \\
 \vr t(b)  \\
      \vdots  \  
\end{bmatrix}
\end{align*}
A solution of a system of equations is any vector $\vr x\in\GLin C$ which is a solution of all equations in the system.
Note that $C$ can be seen as the indexing set of unknowns of the system.

%Likewise, by projecting to a fixed element $c\in C$ we get a \emph{column vector} in $\GLin R$.

For $b\in B$ we denote by $\vr A(b, \_) \in \GLin C$ the row vector indexed by $b$,
and symmetrically, for $c\in C$ we denote by $\vr A(\_, c) \in \GLin B$ the column vector indexed by $c$.
One can also consider the \emph{augmented} matrix $\vr A | \vr t : B\times (C\uplus \set{*}) \tofs \Q$.

In all the examples below let $\Q = \Rat$.

%\PHr{Due to orbit-finiteness of $B$ and $C$, both $\vr A$ and $\vr t$ contain only finitely many distinct numbers/coefficients
%(cf. Lemma~\ref{lem:1O}).
%By scaling one may thus obtain an equivalent system with integer coefficients.}{This fragment seems to be not connected with other part of the text and is 
%distracting.}
\begin{slexample}
Let columns be indexed by $C = \otu \A 2$ and rows by $B = \utu \A 2$.
Consider the system of equations containing, 
for every $\set{\a,\b}\in B$, the equation
$(\constvr 1 {\a\b} + \constvr 1 {\b\a}, 1)$.
Using the formal-sum notation as in~\eqref{eq:formalsum} it may be written as
$(\a\b + \b\a, 1)$ or,
% (we use the formal-sum notation as in~\eqref{eq:formalsum}).
identifying column indexes $\a \b\in C$ with unknowns, as: % $\vr x(\a \b)$, as:
\[
\a\b + \b\a \ = \ 1 \qquad (\a,\b\in\A, \a \neq \b).
\]
All the equations are thus finitary, and the target is $\vr t = \constvr 1 B$.
The constant vector $\vr x = \constvr {\frac 1 2} {} : (\a,\b) \mapsto \frac 1 2$ is a solution.
The system has no finitary solution, as such a solution is in contradiction with the infinitary target $\vr t = \constvr 1 B$.
Furthermore, the system has no integer (infinitary) solution either, as any such solution $\vr x$ would 
necessarily satisfy, for every distinct atoms $\a,\b \in \A\setminus \supp{\vr x}$, the equality $\vr x(\a\b) = \vr x(\b\a)$, 
which is in contradiction
with $\vr x(\a\b) + \vr x(\b\a) = 1$.
\end{slexample}

\begin{slexample} \label{ex:nofin}
Let $C = \otu \A 2$, $B = \A$, and 
consider the system of equations containing, for every $\a\in \A$, the equation
$(\constvr 1 {\a\_}, 1)$.
As before, identifying column indexes $\a \b\in C$ with unknowns, the system may be written as:
\[
\sum_{\b\in\A\setminus\set{\a}} \! \a\b \ = \ 1 \qquad (\a\in \A).
\]
All the equations are thus infinitary.
The system has an integer solution. 
Take any two fixed atoms $\c,\d\in\A$ and consider the vector
\[
\vr x \ = \ \constvr 1 {\_ \c} + \constvr 1 {\c \d}.
\]
Indeed, for $\a\neq \c$ we have
$\innerprod {\constvr 1 {\a\_}}  {\vr x} = \constvr 1 {\a\c} \cdot \constvr 1 {\a\c} = 1$ as required.
Furthermore, for $\a=\c$ we have
$\innerprod {\constvr 1 {\a\_}}  {\vr x} = \constvr 1 {\c\d} \cdot \constvr 1 {\c\d} = 1$ as required.
The system has no finitary solution (essentially for the same reason as in the previous example), and no equivariant one
(as the only equivariant vectors over $C$ are constant ones $q\cdot\constvr 1 {}$, and the inner product
$\constvr 1 {\a\_} \cdot \constvr 1 {}$ is ill-defined for every $\a\in\A$ as long as $q\neq 0$).
\end{slexample}

The two above examples show that the solvability problem is sensitive to additional restrictions on solutions:
the answer changes if solutions are additionally required to be
equivariant, finitary, or integer.
The next example shows that our implicit restriction to finitely-supported solutions also matters:

\begin{slexample} \label{ex:no-sol}
Let $C = \utu \A 2$, $B = \A$, and 
consider the system of equations containing, for every $\a\in B$, the equation
$(\constvr 1 {\set{\a,\_}}, 1)$, where
\[
\set{\a,\_} \ = \ \setof{\set{\a, \g}}{\a\neq\g\in\A}
\]
is the set of all 2-sets containing $\a$.
We argue that the system has no (finitely supported) solution 
(despite the apparent similarity to the system in Example~\ref{ex:nofin}). 
Towards contradiction suppose it has a solution $\vr x$, supported by some $S\subseteqfin \A$.
Thus it is constant on every $S$-orbit in $\utu \A 2$.
An infinite $S$-orbit in $\utu \A 2$ is either the set $\utu {\A\setminus S} 2$
of all 2-sets disjoint from $S$
or, for some fixed $\a\in S$, the set of all 2-sets with one element $\a$ and the other element not in $S$:
\[
\setof{\set{\a,\g}}{\g\in\A\setminus S}
\]
%\PH{Add how S orbits look like, otherwise it is hard to follow}
Therefore each infinite $S$-orbit in $\utu \A 2$ intersects infinitely with $\set{\a,\_}$ for some $\a\in\A$.
In consequence, $\vr x$ is necessarily $0$ when restricted to any infinite $S$-orbit in $\utu \A 2$ as otherwise
$\innerprod {\constvr 1 {\set{\a,\_}}} {\vr x}$ would be ill-defined for some $\a\in\A$.
Therefore $\vr x$ is forcedly finitary, and the argument of the previous examples applies.

On the other hand the system would have an integer solution if we drop the implicit finite-support constraint.
For instance, taking any enumeration $\A = \set{\a_0, \a_1, \a_2, \ldots}$ of atoms, 
the function $\vr x : C\to \Q$ that maps each set
$\set{\a_{2n},\a_{2n+1}}$ to $1$, for $n = 0, 1, \ldots$, and all other sets to $0$, satisfies all equations.
Note that $\vr x$ is not finitely supported, i.e., there is no finite $S\subseteq \A$ such that
$\pi(\vr x) = \vr x$ for all $\pi \in \Aut S$.
\end{slexample}

\para{Solvability of linear equations}

%Let $\X\subseteq \Q$.
We investigate the following type of solvability problems: 
%given an orbit-finite system of linear equation, we ask if it
%has a finitely supported solution using only values from $X$.
%$\vr x : C\tofs\X$.
%\PHr{(recall that we implicitly restrict to finitely-supported solutions only)}{remove pink}:
%

%\label{prb:fin sol}

\prob{\solvname$(\Q)$}
{an orbit-finite system of linear equations} % $(\vr A, \vr t)$} % over an orbit-finite set $C$}
{does it have a solution} % $\vr x : C\tofs \X$}

%The prototypical case is \solvname$(\Q)$, where we allow for all rational solution vectors $C\tofs \Q$, but
%one can consider constrained variants where, for instance, solutions are additionally required to be integer and/or nonnegative:
%\solvname$(\Qplus)$, \solvname$(\Int)$ and \solvname$(\Nat)$.
%%
%
%As we manipulate with solutions using addition and subtraction in our proofs, we have to restrict to subgroups $\X$ of 
%the additive group $(\Q, +)$.
%In particular we cover \solvname$(\Int)$ but not \solvname$(\Qplus)$ or \solvname$(\Nat)$.
%A subgroup $\X$ of $(\Q, +)$ is called \emph{unitary} if $1\in \X$.
%It is \emph{effective}%
%\footnote{Since there are uncountably many different subgroups of $(\Q, +)$~\cite{subgrQ}, some of them are forcedly ineffective.} 
%if it is decidable if a classical finite system of linear equations has a solution in $\X$.

\smallskip

\noindent
As our main result we prove:
\begin{theorem} \label{thm:solv}
\solvname$(\Q)$ is decidable for every fixed effective commutative ring $\Q$.
% any unitary and effective subgroup $\X$ of $(\Q,+)$.
\end{theorem}
The proof, occupying the whole Sections~\ref{sec:fin-solv} and~\ref{sec:solv2fin-solv},
 is by a reduction to solvability of finite systems of linear equations, and
the transformation suffers from a singly-exponential blowup.
As an intermediate step 
we also consider a variant of the problem where solutions are constrained to be finitary, called
\finsolvname$(\Q)$.

\begin{slremark}
In case $\Q = \Rat$, when coefficients in the input system are rational and we seek for rational solutions,
as a corollary of the proof we deduce that
the answer does not change if solutions are relaxed to real ones.
\end{slremark}

\medskip

\para{Spans}
For a subset $P\subseteq \GLin B$, we define $\Span {} P\subseteq \GLin B$ as the set of 
all linear combinations of vectors from $P$, forming a subspace of $\GLin B$:
\begin{align*} %\label{eq:span}
\begin{aligned}
\Span {} P \ = \ \ & \setofbeg{q_1 \cdot \vr p_1 + \ldots + q_k \cdot \vr p_k}{k\geq 0,}\\
& \setofend{\ q_1,\ldots,q_k\in\Q, \ \vr p_1, \ldots, \vr p_k\in P}.
\end{aligned}
\end{align*}
Given a matrix $\vr A \in \GLin {B\times C}$ with rows $B$ and columns $C$, we can define a partial operation of 
multiplication of $\vr A$ by a vector $\vr v\in\GLin C$ in an expected way:
\[
(\mult {\vr A}{\vr v})(b) = \innerprod {\vr A(b, \_)} {\vr v}
\]
for every $b\in B$.
The result $\mult {\vr A} {\vr v} \in \GLin B$ is well-defined if $\innerprod {\vr A(b, \_)}{\vr v}$ is well-defined for all $b\in B$.
%Considering the matrix, via its column-indexed representation $\vr A : C \to \GLin B$, as an indexed set of column vectors
%from $\GLin B$, 
The multiplication $\mult {\vr A}{\vr v}$ can be also seen as an orbit-finite linear combination of 
column vectors $\vr A(\_, c)$, for $c\in C$, with coefficients given by $\vr v$. 
This allows us to define the span of $\vr A$ seen as a $C$-indexed orbit-finite set of 
vectors $\vr A(\_, c) \in \GLin B$:
\begin{align} \label{eq:Gspan}
\GSpan {} {\vr A}  :=  \setof{\mult{\vr A}{\vr v}}{\vr v \in \GLin C, \ \mult{\vr A}{\vr v} \text{ well-defined}}.
\end{align}
The solvability problem for a system of equations $(\vr A, \vr t)$ amounts thus to deciding
if $\vr t \in \GSpan {} {\vr A}$.
When $\vr v$ is finitary, well-definedness is vacuous, and we may define:
%By restricting to finitary vectors $\vr v$ we do not need to assume well-definedness, as it vacuously holds.
%The finitary span of $\vr A$ is thus defined as:
\[
\Span {} {\vr A} \ := \ \setof{\mult{\vr A}{\vr v}}{\vr v \in \Lin C} \ = \ \Span {} P,
\]
%We intentionally overload the notation, as  $\Span {} {\vr A} = $ $\Span {} P$, 
for $P = \setof{\vr A(\_, c)}{c\in C}$ the set of column vectors of $\vr A$.

\para{Outline}

Concerning the proofs, we proceed in three steps.
We start by proving the Orbit-Finite Basis Theorem in Section~\ref{sec:bases},
a crucial technical tool for subsequent steps. % (Theorem~\ref{thm:basis})
As a key novelty, we introduce here the concept of \emph{tight} orbits.  
Then we prove decidability of \finsolvname$(\Q)$ in Section~\ref{sec:fin-solv}, by reducing it to solvability 
of classical finite systems of linear equations.
This step relies on a generalisation of \emph{cogs} introduced in~\cite{BKM21,HR21}.
Finally, 
%Then, as a preparation before proving our main result, 
%in Section~\ref{sec: orb fin sum} we characterize those orbit-finite sets of vectors where the sum is well-defined.
%Using this characterisation, 
in Section~\ref{sec:solv2fin-solv} we reduce \solvname$(\Q)$ to \finsolvname$(\Q)$,
%for $\X \in \set{\Q,\Int}$, 
thus completing the proof of Theorem~\ref{thm:solv}.
This part strongly relies again on the technology developed in Section~\ref{sec:bases}.

%% file: bases-simplified.tex
% !TEX root = main.tex

\section{Proof of the Orbit-Finite Basis Theorem} \label{sec:bases}

In this section we prove Theorem~\ref{thm:basis},
i.e., provide a construction of an orbit-finite basis in $\GLin B$, where $B$ is an arbitrary orbit-finite set.
%The proof builds on some abstract properties of orbits;
%in particular, we introduce a novel concept of \emph{tight} orbits. 
%A more concrete (but not shorter) 
%version of the proof, in the special case when each orbit in $B$ is of the form $\otu \A k$ for some $k\in\Nat$,
%is in the appendix.

\para{Preliminaries}
The mapping $x\mapsto \supp x$ is equivariant:
\begin{claim} \label{claim:pisup}
$\supp{\pi(x)} = \pi(\supp x)$ for every element $x$ and $\pi\in\Aut{}$.
\end{claim}
%\arka{The proofs of the next two claims probably can be sent to appendix}
%
%By the previous claim we deduce:
%\begin{claim}  \label{claim:suppsize}
%The supports of every two elements of the same (not necessarily equivariant) orbit are of the same size. 
%\end{claim}

We rely on the following basic properties of orbits:

%\vspace{-2mm}
\begin{restatable}{claim}{claimOrbitSup}
%\begin{claim} 
\label{claim:orbitsup}
Let $S\subseteqfin\A$.
Each equivariant orbit $O$ 
%be an equivariant orbit \PH{It holds for any orbit i think}. Then $O$ 
contains at most $\size S!$ many elements $x$ with $\supp x = S$.
%\end{claim}
\end{restatable}

\vspace{-2mm}
\begin{restatable}{claim}{claimOrbitFin}
%\begin{claim} 
\label{claim:orbitfin}
Every orbit is either a singleton or an infinite set.
%\end{claim}
\end{restatable}

\begin{definition}
Let $S\subseteqfin\A$.
We define the \emph{$S$-atom dimension} of an $S$-orbit $O$, written
$\dimm S O$, as the size of $\supp x$ for some (every) element $x\in O$, but not counting elements of $S$:
\[
\dimm S O \ := \ \size{\supp x \setminus S}.
\]
The choice of $x$ is irrelevant due to Claim~\ref{claim:pisup}.
When $S$ is clear from the context we omit $S$ and speak of \emph{atom dimension}.
\end{definition}

%\begin{proof}

\para{Reduction to single-orbit $B$}

We claim that we can assume, w.l.o.g., that $B$ is a single orbit.
Indeed, let $T = \supp B$ and let $B = B_1 \uplus \dots \uplus B_n$ be the partition into $T$-orbits.
Then $\GLin B$ is isomorphic to the Cartesian product $\GLin{B_1} \times \dots \times \GLin{B_n}$.
Denote by $\iota_i : \GLin {B_i} \to \GLin B$ the natural embedding that extends a vector $\vr{v} : B_i \tofs \Q$ 
by 0 for all $b\in B\setminus B_i$:
\[
\iota_i(\vr v)(b) \ := \ 
\begin{cases}
\vr v(b) & \text{ if } b\in B_i \\
0 & \text{ otherwise.}
\end{cases}
\]  
Supposing we have orbit-finite bases $\basisof {B_1}, \ldots, \basisof{B_n}$ of the vector spaces $\GLin{B_1}, \ldots, \GLin{B_n}$, respectively,
we get the basis $\basisof B$ of $\GLin B$ as the union of embeddings of $\basisof {B_1}, \ldots, \basisof{B_n}$:
\[
\iota_1(\basisof{B_1}) \ \cup \ \ldots \ \cup \ \iota_n(\basisof{B_n}).
\]
%the union of bases of $\GLin {B_1} \ldots \GLin{B_n}$ gives rise to a base of $\GLin B$.
%\slawek{this 2 sentences added:}
%Each $B_i$, being a $T$-orbit, is $T$-supported ($T_i = \supp {B_i} \subseteq T$) and
%therefore is also a $T_i$-orbit.
We thus assume w.l.o.g.~that  $B$ is a single $T$-orbit. %\smallskip
%\slawek{but it is not necessarily true that $\supp B = T$, we only know that $B$ is supported by $T$.}

As the support of a function is also a support (but not necessarily the support) of its domain, we note:
\begin{claim} \label{claim:supvect}
$T\subseteq \supp{\vr v}$ for every vector $\vr v\in \GLin B$.
\end{claim}

\para{Tight orbits}

A key role is played in the proof by the %(novel, to our knowledge) 
concept of \emph{tight} orbits.
\begin{definition}
Let $S\subseteqfin\A$.
An $S$-orbit $O$ is called \emph{tight} %\slawek{it it the right name?} 
if $S\subseteq \supp x$ for every $x\in O$.
\end{definition}
In particular, every singleton is a tight orbit.
\begin{slexample} \label{ex:tight}
Recall Example~\ref{ex:basis}.
In case of $B = \otu \A 2$, the tight orbits $O\subseteq B$ are the following ones:
\[
%\emptyset\text{-orbit } 
B\qquad
%\set{\a}\text{-orbit } 
\a \_\qquad
\_ \b\qquad
\set{\a \b}
\]
where $\a,\b$ range over atoms and $\a\neq \b$.
The set $B$ is an equivariant orbit, $\a \_$ is an $\{\a\}$-orbit, $\_ \b$ is a $\{b\}$-orbit,
and $\set{\a, \b}$ is an $\{\a, \b\}$-orbit.
Contrarily, for two fixed and distinct $\a,\b\in\A$, the $\set{\a,\b}$-orbit
%\[
%\neq\!\!\a \_ \ = \ \otu \A 2\setminus \a\_,
%\]
%for a fixed $\a\in\A$, is not tight, as well as the orbit
\[
\neq\!\!\a \b \ = \ \setof{\g \b}{\g\notin \set{\a, \b}},
\] 
is not tight.
\end{slexample}

W.l.o.g.~we can assume that $B$ is tight, i.e., $T\subseteq \supp b$ for every $b\in B$. 
Indeed, it is 
sufficient to continue with $B':= B \times \{T\}$. Then $B'$ and $B$ are related by a $T$-supported bijection.
%enough to replace each element $b\in B$ by the pair $(b, T)$.
For future use we state:

%\vspace{-3mm}
\begin{restatable}{claim}{claimTightIso}
%\begin{claim} 
\label{claim:tightiso}
Let $S\subseteqfin\A$.
Every $S$-orbit $O$ is in an $S$-supported bijection with a tight $S$-orbit.
%\end{claim}
\end{restatable}

For every tight $S$-orbit $O\subseteq B$, the size of $S$ is at most the size of the support of elements of $B$.
Furthermore, by Claim~\ref{claim:orbitsup}, for every fixed $S\subseteqfin\A$ there are only finitely many $S$-orbits inside $B$.
In consequence we deduce that the set  of all tight orbits in $B$ is orbit-finite:
\begin{claim}\label{claim:tight}
The set $\setof{O}{O\subseteq B \text{ a tight orbit}}$ is orbit-finite.
\end{claim}
%
%most the size of the support of some (equivalently, all) element $x\in O$.
%
In the sequel we order tight orbits in $B$ with respect to inclusion.

\para{Definition of the basis}

We define $\basisof B$ as the set of characteristic vectors of all tight orbits $O \subseteq B$:
\[
\basisof B \ := \ \setof{ \constvr 1 O }{O\subseteq B \text{ a tight orbit}}.
\]
Once $B$ is fixed, the set $\basisof B$ is orbit-finite due to Claim~\ref{claim:tight}.
Since every singleton is a tight orbit, $\constvr 1 b\in\basisof B$ for every $b\in B$;
informally speaking, $\basisof B$ extends $B$.

\begin{slexample}
Continuing Example~\ref{ex:tight}, where $B = \otu \A 2$, 
the basis vectors are the following ones:
\begin{align*}
\constvr 1 {}\qquad
\constvr 1 {\a \_}\qquad
\constvr 1 {\_ \b}\qquad
\constvr 1 {\a \b},
\end{align*}
for any non-equal $\a,\b\in\A$.
\end{slexample}

It now remains to argue that $\basisof B$
spans the whole space $\GLin B$,
and that it is linearly independent.

\para{Spanning}
Given a subset $S\subseteqfin \A$ such that $T\subseteq S$, we distinguish the set of all tight $S'$-orbits
for $T\subseteq S' \subseteq S$:
\[
\TO T S \ := \ \setof{O}{O\subseteq B \text{ a tight $S'$-orbit}, T\subseteq S'\subseteq S}.
\]
For every fixed $S$ the set $\TO T S$ is finite since, due to Claim~\ref{claim:orbitsup}, 
$B$ includes only finitely many $S'$-orbits for every fixed $S'\subseteqfin \A$.

%
%
%\begin{claim}\label{claim:tightincl}
%For every fixed tight orbit $O$, there are only finitely many tight orbits included in $O$.
%\end{claim}

We prove that $\basisof B$ spans the whole space, i.e., each vector is a linear combination of vectors from $\basisof B$.
%
%, and moreover
%this finite combination is unique.
%%
%\begin{claim} $\GLin B = \Span B C$.
%\end{claim}
%\begin{claimproof}
To this aim we fix a finite subset $S\subseteq \A$ such that $T\subseteq S$ and prove that 
%(cf.~Lemma~\ref{lem:lincomb} below)
every $S$-supported vector $\vr v$ is a linear combination of vectors from 
\[
\basisof B_S \ = \ \setof{\constvr 1 O}{O\in \TO T S} \ \subseteq \ \basisof B.
\]
%, where $\range f = \setof{f(i)}{1\leq i \leq k, \ f(i) \text{ is defined}} \subseteq \A$.
%(\arka{I think it should be $range(f) \subseteq S$.})
For every fixed $S$ the set $\basisof B_S$ is finite, as $\TO T S$ is so.

%By Claim~\ref{claim:Sorbits} and Lemma~\ref{lem:1O} we immediately deduce:
%
%\begin{claim} \label{claim:orb}\PH{It is already stated in Lemma~\ref{lem:1O}}
%Every vector $\vr v$ supported by $S$ 
%is constant when restricted to every $S$-orbit $O$; we may thus write $\vr v(O)$ to denote this constant value.
%\end{claim}
%
\begin{lemma}[Spanning] \label{lem:lincomb}
Let $S \subseteqfin \A$ such that $T\subseteq S$.
Each $S$-supported vector $\vr v\in\GLin B$ is a linear combination of vectors from $\basisof B_S$.
\end{lemma}
\begin{proof}
Let $\vr v\in\GLin B$ and $S\subseteqfin\A$ such that $\supp{\vr v}\subseteq S$. 
By Lemma~\ref{lem:1O}(i), $\vr v$
is constant when restricted to every $S$-orbit $O$; we may thus write $\vr v(O)$ to denote this constant value.
We naturally define the \emph{$S$-orbit-domain} of $\vr v$ as follows:
\[
\odom {\vr v} \ := \ \setof{O}{O\subseteq B \text{ an $S$-orbit}, \ \vr v(O)\neq 0}.
\]
For two $S$-supported vectors  $\vr w, \vr w' \in \GLin B$, we write
$\vr w \prec \vr w'$ if $\odom{\vr w}$ is obtained from $\odom {\vr w'}$ by removing one $S$-orbit
and replacing it by arbitrarily many $S$-orbits of strictly smaller $S$-atom dimension. 

We define a representation of $\vr v$ in basis $\basisof B$  
by structural induction with respect to the transitive closure of $\prec$.
Concerning the induction base, if $\odom {\vr v}$ is empty
% contains only finite orbits, 
then $\vr v$ is the zero vector and the claim holds vacuously.
%follows trivially, since $\basisof B$ extends $B$.
Otherwise, suppose the claim holds for all strictly smaller vectors $\vr w$.
Take an $S$-orbit $O\in\odom{\vr v}$ of maximal $S$-atom dimension. % (hence $O$ is forcedly infinite).
Let 
\begin{align}\label{eq:S'}
S' := \supp x \cap S
\end{align}
for some (every) $x\in O$.
Note that $T\subseteq S'$ as $T\subseteq \supp x$ (since $B$ is tight) and $T\subseteq S$
(by Claim~\ref{claim:supvect}).
We define the $S'$-orbit $O'$ as $S'$-closure of $O$:
\[
O' \ := \ \setof{\pi(x)}{x\in O, \ \pi \in\Aut {S'}}.
\]
By definition, $S'$ is included in the support of every element of $O'$, therefore the orbit $O'$ is tight,
and hence $\constvr 1 {O'}\in \basisof B_S$.
As $S'\subseteq S$, every $S$-orbit in $B$ is either included in $O'$ or disjoint from it, and hence
$O'$ is a finite union of $S$-orbits.
We claim that $O$ has the largest $S$-atom dimension among all $S$-orbits included in $O'$:
\begin{claim} \label{claim:ind}
For every $S$-orbit $M$ included in $O'$ but different than $O$,
we have $\dimm S M < \dimm S O$.
\end{claim}
\begin{claimproof}
%\PH{This proof is written in a strange way. Why not simply take two elements from $O'$ such that their dimension is maximal, and show that 
%there is $S$ permutation that maps one to another?, This shows that they are in the same $S$ orbit so there is only one $S$ orbit of maximal dimension i.e. 
%$O$.}
Recall that $S'\subseteq S\cap \supp x$ for every $x\in O'$.

Consider the subset $N\subseteq O'$ containing those
elements $x\in O'$ for which $S' = \supp x \cap S$.
By the definition of $S'$~\eqref{eq:S'} we have $O\subseteq N$. 
We prove $N\subseteq O$, by showing that
every element $y \in N$ is related by an $S$-atom automorphism to some element of $x\in O$.
Indeed, consider any $x\in O$ and $y = \pi'(x)$ for any $\pi'\in \Aut {S'}$ such that $y \in N$. We have
\[
S' \ = \ \supp x \cap S \ = \ \supp {\pi'(x)} \cap S
\]
and hence there is some $\pi\in \Aut S$, possibly different than $\pi'$, that coincides with $\pi'$ on $\supp x$,
which implies $y = \pi(x)$, as required.
The two inclusions imply $N=O$.

Finally, for all  $x\in O'\setminus N = O' \setminus O$ we have
$S' \subsetneq \supp x \cap S$, which implies that each $S$-orbit $M\subseteq O'$ different than $O$ has strictly
smaller $S$-atom dimension than $O$.
\end{claimproof}

Consider the vector
\begin{align} \label{eq:v'pie}
\vr w \quad := \quad \vr v \ - \ \vr v(O) \cdot \constvr 1 {O'}.
\end{align}
Note that $\vr w$ is supported by $S$ as both $\vr v$ and $\constvr 1 {O'}$ are so, and
$\vr w(O) = 0$.
By Claim~\ref{claim:ind} we infer that $\vr w \prec \vr v$
and therefore by the induction assumption $\vr w$ is a linear combination of vectors from $\basisof B_S$. 
By~\eqref{eq:v'pie} we deduce the same for $\vr v$.
This completes the proof of Lemma~\ref{lem:lincomb}.
 \end{proof}

\para{Linear independence}
We rely on the following property of tight orbits (not true for arbitrary orbits):
\begin{claim} \label{claim:inclcup}
If orbits $O, O_1, \ldots, O_n$ are tight and
$O \subseteq O_1 \cup \ldots \cup O_n$ then
$O\subseteq O_i$ for some $i=1, \ldots, n$.
%$O_1$ is $S_1$-closure of $O$.
\end{claim}
\begin{claimproof}
If $O$ is a singleton then the claim holds vacuously.
Relying on Claim~\ref{claim:orbitfin} we may thus assume that $O$ is infinite.

Suppose $O \subseteq O_1 \cup \ldots \cup O_n$ for a tight $S$-orbit $O$ and arbitrary tight orbits
$O_1, \ldots, O_n$.
Take any $x\in O$ and let $R := \supp x \setminus S$.
Consider elements $\pi(x)\in O$ for all $S$-atom automorphisms $\pi$, thus ranging over all elements of 
the orbit $O$.
At least one of the orbits $O_1, \ldots, O_m$, say the $S_1$-orbit $O_1$, necessarily contains $\pi(x)$ and $\pi'(x)$, for 
some two $S$-atoms automorphisms $\pi, \pi'$, such that the sets $\pi(R)$ and $\pi'(R)$ 
are disjoint.
By tightness of $O_1$ (and relying on Claim~\ref{claim:pisup}) we get $S_1 \subseteq \supp{\pi(x)} = S \cup \pi(R)$ and
$S_1 \subseteq \supp{\pi'(x)} = S \cup \pi'(R)$, and hence $S_1 \subseteq S$, 
which implies  $\pi(x)\in O_1$ for all $S$-atom automorphisms $\pi$, i.e., $O\subseteq O_1$.
\end{claimproof}

We now argue that the set $\basisof B$ is linearly independent.
Towards contradiction, suppose that the zero vector is obtainable as a linear combination of basis vectors
\begin{align} \label{eq:linind}
q_{1} \cdot \constvr 1 {O_1} \ + \ \ldots \ + \  
q_{n} \cdot \constvr 1 {O_n} \quad = \quad \zerovec,
\end{align}
for some tight pairwise-different orbits $O_1, \ldots, O_n \subseteq B$ and 
$q_{1}, \ldots, q_{n} \in \Q\setminus\set{0}$.
%Note that any orbit (and in particular tight one) is either a singleton or an infinite set.
%If all orbits $O_1, \ldots, O_n$ are singletons, the basis vectors $\constvr 1 {O_1}, \ldots, \constvr 1 {O_n}$ belong to $B$,
%in which case~\eqref{eq:linind} is a contradiction.
%Otherwise, 
Take any inclusion-maximal orbit among $O_1, \ldots, O_n$, say $O_1$.
We distinguish two cases.

%\medskip

\paragraph{1} 
If $O_1 \subseteq O_2 \cup \ldots \cup O_n$ then
using Claim~\ref{claim:inclcup} we arrive at a contradiction with the inclusion-maximality of $O_1$.

%\medskip

\paragraph{2} 
Otherwise $O_1 \not\subseteq O_2 \cup \ldots \cup O_n$.
Taking any $x \in O_1 \setminus (O_2 \cup \ldots \cup O_n)$ we derive a contradiction, as the  value 
of the left-hand side of~\eqref{eq:linind}  on $x$ is non-zero:
\[
\left(q_{1} \cdot \constvr 1 {O_1} \ + \ \ldots \ + \  
q_{n} \cdot \constvr 1 {O_n}\right) (x) \ = \ q_1 \ \neq \ 0,
\]
while the value of the right-hand side is $\zerovec(x) = 0$. %, a contradiction.

%% file: fin-solv-ph.tex
% !TEX root = main.tex

\section{Decidability of finitary solvability} \label{sec:fin-solv}

In this section we prove decidability of the finitary solvability problem.

\prob{\finsolvname$(\Q)$}
{an orbit-finite system of linear equations} % over an orbit-finite set $C$}
{does it have a finitary solution} % $\vr x : C\tofin \X$}

\begin{theorem} \label{thm:fin-solv}
\finsolvname$(\Q)$ is decidable for every fixed effective commutative ring $\Q$.
% unitary and effective subgroup $\X$ of  $(\Q, +)$.
\end{theorem} 

%Recalling Remark~\ref{rem:anyF}:
%the proof works also for any other fixed field in place of $\Q$, for instance for every finite field.
%\slawek{is this recall needed?}\PH{Actually theorems should be formulated in the most general way that it is possible to recall them in other works, so 
%probably the statement of the theorem should be modified.}

Let $\vr A \in \GLin{B \times C}$ and $\vr t\in\GLin B$ be the input.
We need to check if $\vr t \in$ $\Span {} {\vr A}$, or equivalently
$\vr t \in \Span {} P$, where $P = \setof{\vr A(\_, c)}{c\in C}$ is an orbit-finite set of vectors from $\GLin B$.
As $P$ can be computed from $\vr A$, 
from now on we assume we are given $P$ and $\vr t$.
% concentrate on the latter question.

\para{Simplifying assumptions}

%W.l.o.g.~we may assume that  $B$
%is equivariant---%
%otherwise extend $B$ to its equivariant closure $\setof{\pi(b)}{b\in B, \pi\in\Aut{}}$, 
%end correspondingly extend vectors in $P$ and $\vr t$ with zeros.
First, for simplicity of presentation
%in order to avoid additional secondary technicalities
we assume that $P$ (but not $\vr t$) is equivariant; hence also $B$ is forcedly so.
%\slawek{how to justify this?}
%
% FINAL VERSION: removed the next sententce;
%
%This is w.l.o.g.~as equivariance of $P$ can be assured exactly like  in the transformation to canonical 
%form (cf.~\eqref{eq:newcan} on page~\pageref{eq:newcan})  below. \slawek{Avoid forward reference!}
%
%
%
% reduces (both problems are actually inter-reducible) to the following \emph{finitary spanning problem}:
%%
%\probnoname
%{an orbit-finite set $B$, an orbit-finite subset $P\subseteq \GLin B$, a vector $\vr t\in\GLin B$}
%{does $\vr t\in\Span {} P$}
%
%\noindent
%%Reductions in both way are based on effective transformations discussed in Remark~\ref{rem:matrix-repr}.
%%In one direction, given an $R$-indexed system of linear equations over $C$, one computes the target vector 
%%$\vr t\in \GLin R$ and the set $P$ of column vectors (indexed by $C$).
%%In the reverse direction, given an orbit-finite set $R$, an orbit-finite (say $C$-indexed) subset $P\subseteq \GLin R$, 
%%and a vector $\vr t\in\GLin B$,
%%one computes a set of equations over $P$, indexed by $B$, by projecting $P$ and $v$ to each element $b \in B$.
%%
%Indeed, given a matrix $\vr A : C \to \GLin B$ and a vector $\vr t$, 
%it is enough to compute the set of column vectors of $\vr A$:
%$P = \setof{\vr A(c)}{c\in C}$.
%Then the system $(\vr A, \vr t)$ has a solution if, and only if $\vr t \in \Span {} P$.
%%
%
%For the ease of presentation we focus on the case $\X = \Q$ 
%(the adaptation to arbitrary $\X$ is discussed at the end of this section).
%
%\TODO{should we expand on this?}
%As a preparatory step we observe that 

We further assume w.l.o.g.~that 
all vectors are finitary: $P\subseteq \Lin B$ and $\vr t\in\Lin B$.
Indeed, 
%given the input consisting of an orbit-finite set $B$, an orbit-finite set $P\subseteq\GLin B$, 
%and $\vr t\in\GLin B$, 
according to Remark~\ref{rem:effect}
%using Corollary~\ref{cor:basis} 
we may compute an orbit-finite basis $\basisof B$ of $\GLin B$, and then compute 
the representations $P'\subseteq \Lin {\basisof B}$ and $\vr t'\in\Lin{\basisof B}$ of $P$ and $\vr t$ in this basis.
As $\basisof B$ is a basis, the representation preserves solvability: $\vr t\in\Span {} P$ if, and only if $\vr t'\in\Span{B'}{P'}$.
%\PH{Here it is important to remark that as we use fin-solve everything is well defined.}

Finally, we assume w.l.o.g.~that $B$ is \emph{straight}, by which we mean that
each of its orbits
is in equivariant bijection with  $\otu \A k$ for some $k\in\Nat$.
By Lemma~\ref{lem:repr}, each (equivariant) orbit in $B$ is in equivariant bijection with
$\faktor{\A^{(k)}}{G}$ for some $k \in \Nat$ and some subgroup $G\leq \symg k$.
The vector space $\GLin{\faktor{\otu \A k}{G}}$ is, in turn, in equivariant bijection 
with the subspace of all \emph{$G$-invariant} vectors in
$\GLin{\otu \A k}$, i.e.~vectors $v : \otu \A k \to \Q$
satisfying $v(a \circ \sigma) = v(a)$ for every $a\in \A^{(k)}$ and $\sigma\in G$.
This yields the embedding
% \arka{Why the word \emph{bijective}?}
\[
\iota : \GLin{\faktor{\otu \A k}{G}} \to \GLin{\otu \A k}
\]
given by pre-composing with the canonical quotient 
$\quot G : \otu \A k \to \faktor{\otu \A k}{G}$,
\begin{align} \label{eq:emb}
\vr v \quad \mapsto \quad \iota(\vr v) \ = \ \vr v \circ \quot G.
\end{align}
The embedding $\iota$ extends to $\GLin B \to \GLin{B'}$, where $B'$ is the disjoint union of straight orbits 
corresponding to orbits of $B$.
The embedding is efficiently computable, and preserves linear combinations and finitariness.
By the latter property we may restrict $\iota$ to finitary vectors, namely
$\iota : \Lin B \to \Lin {B'}$.
Therefore, writing $P'$ and $\vr t'$ for $\iota(P)$ and $\iota(\vr t)$, respectively, we deduce that
$\vr t \in \Span {} P$ if and only if $\vr {t'} \in \Span {} {P'}$.

Summing up, by an \emph{instance} of the problem we mean a triple $(V, P, \vr t)$ consisting of
a vector space $V = \Lin B$ generated by an equivariant straight orbit-finite set $B$,
an equivariant orbit-finite subset  $P \subseteq V$, and a vector $\vr t\in V$.
The instance is \emph{solvable} if $\vr t \in \Span {} {P}$.
%\TODO{adapt to $\X$}
%

\para{Canonical form}
Recall that the atom dimension of the orbit $\otu \A k$ is $k$.
Up to an equivariant bijection, we may present $B$ as a disjoint union 
$B = B_1 \uplus \ldots \uplus B_n$ where $B_i = \otu \A {p_i}$ for some $p_i\in\Nat$, for $i = 1,\ldots, n$.
Therefore the vector space $\Lin B$ is equivariantly isomorphic to
\[
(\otu \A {p_1} \tofin \Q) \times \ldots \times (\otu \A {p_n} \tofin \Q).
\]
For convenience we prefer to work with vector spaces in the following canonical form, where 
all orbits $\otu \A {p}$ of the same atom dimension $p$ are grouped together:
\begin{align} \label{eq:canform}
V \ = \ (\otu \A {k_1} \tofin \Q^{\ell_1}) \times
\ldots \times
(\otu \A {k_m} \tofin \Q^{\ell_m}),
\end{align}
where $k_1, \ldots, k_m$ are pairwise different nonnegative integers, and $\ell_1, \ldots, \ell_m$ are arbitrary positive integers.
A vector space $V$ in canonical form~\eqref{eq:canform} is thus the Cartesian product of $m$ \emph{components}.
The definition of  domain naturally extended to vectors of the form $\vr v : \otu \A k \tofs \Q^\ell$ as follows:
\[
\dom{\vr v} = \setof{a\in\otu \A k}{\vr v(a) \neq (0, \ldots, 0) \in \Q^\ell}.
\]
%\[
%V \ = \ \comp V 1 \times \ldots \times \comp V m,
%\]
%where $\comp V i = \otu \A {k_i} \tofs \Q^{\ell_i}$.
\begin{definition} \label{def:ad}
By the \emph{atom dimension} of a vector space $V$ in canonical form~\eqref{eq:canform} 
we mean the maximum among atom dimensions 
of orbits $\otu \A {k_i}$, i.e., $\max(k_1, \ldots, k_m)$.
\end{definition}
The component ${\comp V i = \otu \A {k_i} \tofs \Q^{\ell_i}}$ of largest atom dimension 
we call the \emph{main component} of $V$ and denote as $\main{V}$.
Assuming w.l.o.g.~$i = 1$ (the main component is the first one) we may write
\[
V \ = \ \main{V} \times V'
\]
where $V'$ is the Cartesian product of all non-main components.
Thus every vector $\vr v \in V$ decomposes as a pair 
\begin{align}\label{eq:dec}
\vr v = (\main{\vr v}, \vr v') \in \main V\times V'.
\end{align}
%One may project a vector  $\vr v \in V$ onto the main component, yielding  $\main{\vr v} \in \main V$.
Furthermore,
$\main V$ embeds into $V$ as the subspace 
$\main V \times \set{\zerovec}\times \ldots \times \set{\zerovec}$, where $\zerovec : \otu \A {k_i} \tofs \Q^{\ell_i}$ maps
every tuple $a\in\otu \A {k_i}$ to $(0, \ldots, 0) \in \Q^{\ell_i}$,
and likewise $V'$ embeds into $V$.
Using the embeddings implicitly, we may write
\begin{align} \label{eq:plus}
\vr v \ = \ \main{\vr v} + \vr v'
\end{align}
in place of~\eqref{eq:dec}.

Summing up, instances $(V, P, \vr t)$ are assumed from now on to consist of 
a vector space $V$ in canonical form~\eqref{eq:canform}.

%and hence the matrix $\vr A : B \to \GLin C$ and the right-hand side vector $\vr t \in\Lin B$ of the following form:
%
%\begin{minipage}{0.5\linewidth}
%\[
%\begin{bmatrix}
%\vr A_1 \\
%\ldots \\
%\vr A_n \\
%\end{bmatrix}
%\]
%\end{minipage}
%\begin{minipage}{0.5\linewidth}
%\[
%\begin{bmatrix}
%\vr t_1 \\
%\ldots \\
%\vr t_n \\
%\end{bmatrix}
%\]
%\end{minipage}
%%
%
%\medskip
%\noindent
%for sub-matrices $\vr A_i : B_i\to\Lin C$ and sub-vectors $\vr t_i \in \Lin {B_i}$.
%The maximal atom dimension $k = \max(k_1, \ldots, k_n)$ of orbits $B_1, \ldots, B_n$ we call the atom dimension of $B$.

\para{Locally solvable instances}
We distinguish \emph{locally solvable} instances $(V, P, \vr t)$, defined as follows.
Let $\main V = \otu \A k \tofs \Q^\ell$ be the main component (for succinctness of notation  
we write $k, \ell$ instead of $k_1, \ell_1$).
We use the restriction operation: for $X\subseteq \otu \A k$ and $\vr w : \otu \A k \tofs \Q^\ell$ we define
\[
\restr {\vr w} X(a) = \begin{cases}
\vr w(a) & \text{ if } a \in X\\
\zerovec & \text{ otherwise.}
\end{cases}
\]
Given a $k$-set $A\in\utu \A k$,  we may consider the \emph{$A$-restriction} 
$(\main V, P', \vr t')$
%\arka{$(\main{\vr V}, P', \vr{t'})$?}\PH{Agree with Arka $P'\in \main{\vr V}$ and not in $\vr V$} 
of the instance, where
%$V' = \otu A k \tofs \Q^\ell$ and
\[
P' = \setof{\restr{\main{\vr v}} {\otu A k}}{\vr v \in P} \qquad
\vr t' = \restr{\main{\vr t}} {\otu A k}.
\]
Thus the $A$-restriction is essentially a finite system of at most $\size{\otu A k} = k!$ equations.
Any restriction of a solvable instance is solvable too.
An instance is called \emph{locally solvable} if each of its $A$-restrictions is solvable, for every 
$A\in \utu \A k$.
Clearly, each solvable instance is locally solvable, but the opposite implication is not true in general
(one of the reasons is that local solvability only refers to the main component).
%(for instance, the system in Example~\ref{ex:no-sol} is locally solvable but not solvable). 
%\arka{There is an intermediate step to see this, which is to use the base theorem
%to make the columns finite. That might make it harder for people to understand it.
%Said that probably locally solvable doesn't imply solvable is probably very easy to see
%and doesn't need any example. One could also create some trivial example where the
%non-main components are zero in the vectors in $P$, and zero in $\vr{t}$}\PH{I agree with Arka going through base theorem here is not the best idea}

\begin{restatable}{claim}{claimLocSolvDec}
%\begin{claim} 
\label{claim:loc-solv-dec}
Local solvability is decidable. 
%\end{claim}
\end{restatable}

We later make use of the fact that for any two different (but not necessarily disjoint) $k$-sets
$A, A'\in\utu \A k$, the sets
$\otu A k$ and $\otu {(A')} k$ are always disjoint.
%\PH{I think it will be good to add here a claim claim that local solvability is decidable, otherwise this paragraph is incomplete. As a reader I would be 
%puzzled by this. The proof can be postponed.}

%
%We focus on finite subsets of the form $X = \otu A k$, where $A\subseteq \A$ is a finite subset of atoms of size $k$.
%Such subsets, containing all permutations of $k$ given atoms, 
%we call \emph{singular}, and restrictions to such subsets we call singular restrictions.
%Every singular subset has exactly $k!$ elements, and each two singular subsets are disjoint.
%An instance $(B, P, \vr t)$ is called \emph{non-trivial} if its every singular restriction is solvable; otherwise it
%is \emph{trivial}.
%Trivial instances are forcedly unsolvable.

\para{Reduction of atom dimension}
The following lemma is the core of the proof of Theorem~\ref{thm:fin-solv}:
\begin{lemma} \label{lem:decr-dim}
Given a locally solvable instance $(V, P, \vr t)$ as above,
one may construct another %locally solvable 
instance $(\mdf V, \mdf P, \mdf{\vr t})$ 
%\PH{It may be not locally solvable, otherwise local solvability would be equivalent to solvability}
where atom dimension of $\mdf V$ is strictly smaller  
%\PH{or atom dimension is the same} 
than that of $V$, and such that $\vr t \in \Span {} P$ if and only if \,$\vr {\mdf t} \in \Span {} {\mdf P}$.
\end{lemma}
%
%\PH{I think this description is confusing. The equations that appear are equations that capture local solvability, then you reduce dimension, and a new bunch 
%of equations appears for the level solvability for the lower dimension, and you reduce one more time. You can create a big system of equations, and then solve 
%it, but is much less natural way of thinking about this algorithm. I can rewrite this description if you want.}
\begin{proof}[Proof of Theorem~\ref{thm:fin-solv}]
Using the lemma we prove that the finitary spanning problem 
reduces to solvability of finite systems of linear equations, 
which implies decidability.
First, local solvability of an instance is a necessary condition for solvability, and is decidable by Claim~\ref{claim:loc-solv-dec}.
The algorithm thus checks if the input instance is locally solvable: if it is not so it answers negatively, 
and if it is so the algorithm applies the construction of Lemma~\ref{lem:decr-dim} to produce an instance
of strictly smaller atom dimension.
Continuing so iteratively, the algorithm finally
arrives at $V$ of atom dimension equal to $0$, i.e., at a finitely dimensional vector space $V$.
In this case the set $P$, being an orbit-finite subset of $V$, is necessarily finite too, and the problem amounts to solving a finite system of linear equations.
%(doable in polynomial time).
\end{proof}

We thus concentrate from now on on proving Lemma~\ref{lem:decr-dim}.

%\para{Cogs}
%We recall the concept of \emph{cogs} from~\cite{}.
%Let $a = (a_1, a_2, \ldots, a_n), b = (b_1, b_2, \ldots, b_n)\in \otu \A n$ be two sequences of non-repeating tuples such that
%the supports $\set{a_1, a_2, \ldots, a_n}$ and $\set{b_1, b_2, \ldots, b_n}$ are disjoint.
%Given a subset $I\subseteq \set{1, 2, \ldots, n}$ of indices, we define a tuple
%$
%\mixt a I b = (e_1, e_2, \ldots, e_n) \in \otu \A n
%$
%where 
%\[
%e_i = \begin{cases}
%a_i & \text{ if } i \notin I \\
%b_i & \text{ if } i\in I.
%\end{cases}
%\]
%Intuitively, the set $I$ specifies indices $i$ where $a_i$ should be replaced by $b_i$ in $\mixt a I b$.
%In particular, $\mixt a{\emptyset}b = a$ and $\mixt a{\set{1,2,\ldots, n}} b = b$.
%A \emph{cog} over $a, b$ is a vector $\cog a b \in \Lin {\otu \A n}$ defined as follows:
%\[
%\cog a b = \setof{\mixt a I b}{I\subseteq \set{1,2,\ldots, n}}, \qquad\qquad
%\cog a b(\mixt a I b) = (-1)^{|I|}.
%\]
%In particular, $\cog a b(a) = (-1)^0 = 1$ and $\cog a b(b) = (-1)^n$.

\para{Cogs}
We rely on a generalisation of cogs in~\cite{BKM21}
and of simple hypergraphs in~\cite{HR21}.
Let $A, S\in\utu \A k$ be two disjoint subsets of atoms of size $k$, and
let $\sigma : A \to S$ be a bijection.
For every $I\subseteq A$, we define 
%
%Let $v\in\main V = \otu \A k \tofs \Q^\ell$, and
%let $a = (a_1, a_2, \ldots, a_k), s = (s_1, s_2, \ldots, s_k)\in \otu \A k$ 
%be two sequences of non-repeating tuples such that
%the supports $A = \set{a_1, a_2, \ldots, a_k}$ and $S = \set{s_1, s_2, \ldots, s_k}$ are disjoint.
%Given a subset $I\subseteq \set{1, 2, \ldots, k}$ of indices, we define 
an injective mapping
\begin{align} \label{eq:sigmaI}
\sigma_I : A \to A\cup S
\qquad\qquad
\sigma_I(\a) = 
\begin{cases}
\a & \text{ if } \a \notin I \\
\sigma(\a) & \text{ if } \a\in I.
\end{cases}
\end{align}
Intuitively, the set $I$ specifies those elements $\a\in A$ that should be replaced by $\sigma(\a)$.
In particular, $\sigma_\emptyset$ is the identity on $A$ and $\sigma_A = \sigma$.  %{\set{1,2,\ldots, k}} = \sigma$.
Let $\vr w : \otu \A k \tofin \Q^\ell$ be a vector satisfying $\dom{\vr w} \subseteq \otu A k$.
%\PH{Do we need this $dom(w)\subset A^{(k)}$} 
%\slawek{yes. Thanks to this it does not matter how $\sigma_I$ is extended to $\A$}
In~\eqref{eq:cog} below we implicitly extend $\sigma_I$, in an arbitrary way, to an atom automorphism $\A\to\A$.
A \emph{cog of $\vr w$ via $\sigma$} is the vector $\gcog {\vr w} \sigma : \otu \A k \tofin \Q^\ell$ defined as:
\begin{align}\label{eq:cog}
\gcog {\vr w} \sigma \ = \ \sum_{I\subseteq A} (-1)^{|I|} \cdot \sigma_I (\vr w).
\end{align}
Thus the domain of $\gcog {\vr w} \sigma$ is a finite set of size at most $k!\cdot 2^k$.

\begin{slexample}
Let $k = 2$, $\ell = 1$, $A = \set{\a,\b}\subseteq \A$, and
\[ \vr{w} \ = \ \a\b + 2 \cdot \b\a \ \in \ \otu \A 2 \tofin \Q \ .\]
Let $\sigma : \set{\a,\b} \to \set{\c,\d}$ be defined by $\sigma(\a) = \c$ and $\sigma(\b) = \d$.
% the atom automorphism which swaps $a$ with $c$ and $b$ with $d$, and fixes all other atoms.
Then we have
\begin{align*}
\gcog{\vr{w}} \sigma \ =  \ \quad & \a\b + 2 \cdot \b\a
                       \ \  - \   \c\b - 2 \cdot \b\c \\
                        \, + \,\,& \c\d + 2 \cdot \d\c
                        \ \ \ -  \ \a\d - 2 \cdot \d\a.
\end{align*} 
\vspace{-2mm}
\end{slexample}

\begin{claim} \label{claim:cogw}
Let $\vr w : \otu \A k \to \Q^\ell$ such that $\dom{\vr w} \subseteq \otu A k$.
Then $\restr{\big(\gcog {\vr w} \sigma\big)} {\otu A k} = \vr w$.
\end{claim}

%\begin{claim}
%Every $\sigma_I : A \to A\cup S$, as defined in~\eqref{eq:sigmaI}, 
%extends to an atom automorphism that is identity outside of $\setof{a\in A}{\sigma_I(a)\in S} \cup
%\setof{\sigma_I(a)\in S}{a\in A, \sigma_I(a)\in S}$.
%\end{claim}
%%
%\begin{claimproof}
%
%\end{claimproof}

\begin{proof}[Proof of Lemma~\ref{lem:decr-dim}]
Consider some locally solvable instance $(V, P, \vr t)$ with $V$ in canonical form~\eqref{eq:canform}.

We start by restricting the set $P$ to a subset $P'\subseteq P$ while preserving solvability.
Let $S\in\utu \A k$ be an arbitrary fixed subset of atoms of size $k$ disjoint from $T = \supp {\vr t}$.
For any $p, q \in\Nat$ and $X\subseteq \otu \A p$
let
$X \magicto \Q^q$ denote the subspace
\[
X \magicto \Q^q \ := \ \setof{\vr w : \otu \A p \tofin \Q^q}{\dom{\vr w} \subseteq X}.
\]
%\arka{Why use $\magicto$ instead of $\to$?}
%\slawek{Because the domains $\Aul {p} {\neg S}$ and $\otu \A p$ differ}
Furthermore, let 
\[
\Aul  p {\neg S} \ = \ \otu {(\A\setminus S)} p
\]
denote the set of non-repeating $p$-tuples containing \emph{no} element of $S$, and 
define the subspace $V_S$ of $V$:
\begin{align} \label{eq:VS}
V_S  =  
({\Aul {k_1} {\neg S}} \magicto \Q^{\ell_1}) 
\ \times \ldots \ \times
({\Aul {k_m} {\neg S}} \magicto \Q^{\ell_m}).
\end{align}
Thus $V_S$ contains only those vectors in $V$ whose support is disjoint from $S$.
Consider an instance $(V, P', \vr t)$ where 
$$
P'   \ := \ P \cap V_S \ = \ \setof{\vr v\in P}{\supp {\vr v} \cap S = \emptyset}.
$$
We observe that any finitary solution 
\[
q_1 \cdot \vr v_1 + \ldots + q_m \cdot \vr v_m \ = \ \vr t
\]
of $(V, P, \vr t)$, where $q_1, \ldots, q_m \in \Q$ and $\vr v_1, \ldots, \vr v_m \in P$, 
may be renamed, using a $T$-atom automorphism $\pi$, to a solution involving only
vectors $\pi(\vr v_1), \ldots, \pi(\vr v_m) \in P$ with support disjoint from $S$. 
We have thus argued that:
\begin{claim} \label{claim:instanceS}
For every $S\subseteqfin\A$ disjoint from $\supp {\vr t}$, the instance
$(V, P, \vr t)$ is solvable if and only if $(V, P\cap V_S, \vr t)$ is so.
\end{claim}
\noindent
%\arka{The use of $P'$ and $P \cap V_S$ is confusing. Either use one of them or mention somewhere that $P' = P \cap V_S$.}
The instance $(V, P', \vr t)$ is forcedly locally solvable, and computable from $(V, P, \vr t)$.
%Note that the support grows: $\supp{P'} = \supp P \cup S$.

\para{The instance $(\mdf V, \mdf P, \mdf{\vr t})$}
Let $\main V = \otu \A {k_1} \tofs \Q^{\ell_1}$ be the main component of $V$ 
(for succinctness of notation we write $k$, $\ell$ in place of $k_1, \ell_1$).
For any $p, q \in\Nat$, 
let 
\[
\Aul p {S} \ = \ \otu \A p \setminus \Aul p {\neg S}
\]
denote the set of non-repeating $p$-tuples containing \emph{at least} one element of $S$.
%, and
%let $\Aul p S \magicto \Q^q$ denote the subspace
%\[
%\setof{\vr w : \otu \A p \tofin \Q^q}{\dom{\vr w} \subseteq \Aul p S}.
%\]
%The set $\Aul p S \tofs \Q^q$, resp.~$\Aul  p {\neg S} \tofs \Q^q$, we identify below with the subset of $\otu \A p \tofs \Q^q$
%containing vectors $\vr w$ with $\dom{\vr w} \subseteq \Aul p S$, respectively $\dom{\vr w} \subseteq \Aul p {\neg S}$.
%By definition of $P'$ we have:
%\begin{align*}
%P' \ \subseteq \  ({\Aul {k} {\neg S}} \magicto \Q^{\ell}) \ \times \ 
%({\Aul {k_2} {\neg S}} \magicto \Q^{\ell_2}) & \\
%\ \times \ldots \ \times
%({\Aul {k_m} {\neg S}} \magicto \Q^{\ell_m})&.
%\end{align*}
We define $\mdf V$ as 
the subspace of $V$ where the domain in the main (first) component is included in $\Aul {k_1} S = \Aul k S$,
and in all other components in $\Aul {k_i} {\neg S}$, for $i>1$:
\begin{align} \label{eq:mdfV}
\begin{aligned}
\mdf V \ = \  ({\Aul {k} S} \magicto \Q^{\ell})\  \times \ 
({\Aul {k_2} {\neg S}} \magicto \Q^{\ell_2})  & \\
 \ \times \ \ldots \ \times \ 
({\Aul {k_m} {\neg S}} \magicto \Q^{\ell_m})&. 
\end{aligned}
\end{align}
%\arka{Contradictory use of the notation $\A_{\neg S}$.}
%
%\[
%\mdf V \ = \ \setof{\vr v\in V}{\dom{\main {\vr v}}\subseteq \A_S}.
%%{\A_{S}} \tofs \Q^\ell.
%\]
Formally speaking, the space $\mdf V$
% is a subspace of $V$.
%
%Equivalently, a vector $\vr v \in V$ belongs to $\mdf V$ if $\main{\vr v}(a) = \vr 0$ for all tuples $a\in\A_{\neg S}$.
%
%Recall that the support of a tuple $a = (a_1, \ldots, a_n) \in\otu \A n$, is exactly the atoms appearing in the tuple: 
%$\supp a = \set{a_1, \ldots, a_n}$.
%According to our assumptions $B$ has the following form:
%\[
%B = \biguplus_{i = 1 \ldots m} B_i,
%\]
%where $B_i =\otu \A {k_i}$. 
%Let $k = \max(k_1, \ldots, k_m)$ be atom dimension of $B$ and let
%$B'\subseteq B$ be the union of sets $B_i$ of maximal atom dimension:
%\[
%B' = \biguplus_{i \,:\, k_i = k} B_i.
%\]
%Let $s = (s_1, \ldots, s_k) \in \otu {(\A \setminus T)} k$ be an arbitrary fixed tuple of $k$ atoms outside of $T = \supp{\vr t}$
%and let $S = \supp s = \set{s_1, \ldots, s_k}$.
%Finally, let $B''$ be the subset of $B'$ containing only $S$-fresh tuples
%\[
%B'' = \setof{a \in B'}{\supp a \cap S = \emptyset}.
%\]
%which allows us to define $\mdf B$ as $\mdf B = B - B''$.
%Intuitively speaking, all tuples in $\mdf B$ have either length strictly smaller than $k$,
%or length exactly $k$ but contain at least one element of $S$.
is not in canonical form and it is not even clear how its atom dimension would be defined.
%, once atoms from $S$ are eliminated, 
The canonical form may be easily recovered by "eliminating" atoms from $S$.
This is tackled formally below.

We now proceed to defining $\mdf P$ % = \setof{\mdf{\vr v}}{\vr v\in P'}$ 
and $\mdf{\vr t}$.
Note that for every vector $\vr v \in P'$, the domain of its main component $\main{\vr v}$ is included in $\Aul {k} {\neg S}$.
Our aim is to replace every vector $\vr v \in P'$ by a finite set of vectors $\mdf {\vr v}$ whose domain,
after projecting to the main component, is disjoint from $\Aul k {\neg S}$.
Likewise we aim at replacing $\vr t$ by a vector $\mdf{\vr t}$, while preserving solvability.

In the sequel we fix an arbitrary total order on $S$.
Let $\Ord$ denote the set of all total orders $\prec$ on $\A\setminus S$.
Given an order $\prec$ in $\Ord$,
for every $k$-set $A\subseteq \A\setminus S$ the restriction of $\prec$ to $A$ 
uniquely induces an (order preserving) bijection $\sigma^\prec_A : A \to S$.
For a finitary vector $\vr w \in \main V = \otu \A k \tofs \Q^\ell$ we define a finitary vector $\Deltaw {\vr w}\in\main V$ as follows:
%
%\begin{align} \label{eq:delta1}
%\Delta \vr v = \sum_{a\in \dom{\vr v} \cap B''} \vr v(a) \cdot \cog a s.
%\end{align}
%
%Let $\Pi_{S}$ denote the set of all bijections $\sigma : A\to S$,
%for some (forcedly $k$-element) subset $A\subseteq\A$ disjoint from $S$.
%
\begin{align} \label{eq:delta1}
\Deltaw {\vr w} \ = \ \sum_{A\subseteq \A \setminus S, \size A = k} \gcog {\restr{\vr w} {\otu A k}} {\sigma^\prec_A}.
\end{align}
%{
%\begin{align} \label{eq:delta1}
%\Deltaw {\vr w} \ = \ \sum_{\otu{A}{k}\in \otu {(\A \setminus S)}{k}} \gcog {\restr{\vr w} {\otu A k}} {\sigma^\prec_A}.
%\end{align}
%Otherwise $A$ ranges over sets of sizes different than $k$, and the expression is undefined.
%}
%(where $A$ ranges over $k$-element subsets of $\A\setminus S$).
%
The sum is infinite but well-defined for finitary vectors $\vr w\in \mdf V$, 
as only finitely many cogs $\gcog {\restr{\vr w} {\otu A k}} {\sigma^\prec_A}$ are non-zero, namely 
only when $\otu A k \cap \dom{\vr w} \neq \emptyset$.
For every $\prec \; \in \Ord$, 
the function $\vr w \mapsto \Deltaw {\vr w}$ is a linear mapping (from $\main V$ to $\main V$), and in consequence
so is the function $\vr v\mapsto \vr v - \Deltaw{\main{\vr v}}$:
\begin{claim} \label{claim:linmap}
For every $\prec \; \in \Ord$,
the function $\vr v \mapsto \vr v - \Deltaw{\main{\vr v}}$ is a linear mapping from $V$ to $\mdf V$.
\end{claim}
Using Claim~\ref{claim:cogw} we observe that
$\Deltaw \vr w(a) = \vr w(a)$ for every $a\in \Aul k {\neg S}$.
We define
\begin{align} \label{eq:delta2}
\mdf {\vr v} = \setof{\vr v - \Deltaw {\main{\vr v}}}{\prec \; \in \Ord}
\end{align}
and derive, using the above observation:
\begin{claim}
For every $\vr v \in P'$ we have $\mdf {\vr v} \subseteq \mdf V$.
\end{claim}
Since all vectors $\vr v \in P'$ are finitary,
the set $\mdf {\vr v}$ is finite for every $\vr v \in P'$, even if $\prec$ ranges in~\eqref{eq:delta2} over
all uncountably many total orders $\prec \; \in \Ord$.

We define  $\mdf P := \bigcup_{\vr v \in P'} \mdf{\vr v}$ and derive $\mdf P \subseteq \mdf V$ by the last claim.
We also define $\mdf{\vr t} = \vr t - \Deltawp 0 {\main{\vr t}}$ for some fixed arbitrarily chosen total order $\prec_0 \; \in \Ord$.
We observe that the mapping $\vr v \mapsto \mdf {\vr v}$ is supported by $S$, 
since the set $\Ord$ of total orders is supported by $S$.
In consequence, $\mdf P$ is supported by $\supp {P'} = S\cup \supp P$.
As an orbit-finite union of orbit-finite sets is always orbit-finite~\cite[Exercise~62, Sect.~3]{atombook}, 
so is also an orbit-finite union of finite sets, and we have:
\begin{claim}  \label{claim:Po-f} %\slawek{is this enough?}
$\mdf P$ is orbit-finite.  % and computable from $P$.
\end{claim}
Refering to Remark~\ref{rem:effect} we may state:
\begin{claim} \label{claim:comput}
$(\mdf V, \mdf P, \mdf {\vr t})$ is computable from $(V, P, \vr t)$.
\end{claim}

%\PH{What is atom dimension of $(\mdf V, \mdf P, \mdf {\vr t})$?.}
%\slawek{Def.~\ref{def:ad}}

%\arka{I don't think it is clear that it is orbit-finite.}
%\slawek{orbit-finite union of orbit-finite sets is always orbit-finite}
%
%Finally, we claim:
%\begin{claim}
%The instance $(\mdf V, \mdf P, \mdf {\vr t})$ is locally-solvable.
%\end{claim}
%%
%\begin{claimproof}
%Consider a fixed $k$-element set $A\subseteq \A \setminus S$.
%Since 
%\[
%\restr {\vr t} {\otu A k} \ \in \ \Span {} {\setof{\restr {\vr v} {\otu A k}}{\vr v \in P'}}
%\]
%by local solvability of $(V, P, \vr t)$,
%and $\restr {\mdf {\vr t}} {\otu A k} = \restr {(\vr t - \Deltawp 0 \main {\vr t})} {\otu A k}$,
%\PH{But $\mdf t(\otu A k) =0$ I think ? Only when $S$ is involved it is nonzero} 
%by Claim~\ref{claim:linmap} we also have:
%\[
%\restr {\mdf {\vr t}} {\otu A k}  \ 
%%= \restr {(\vr t - \Deltawp 0 \main {\vr t})} {\otu A k} \ 
%\in \  \Span {} {\setof{\restr {(\vr v - \Deltawp 0 \main{\vr v})} {\otu A k}}{\vr v \in P'}},
%\]
%as required.\PH{I don't get this proof.}
%\end{claimproof}

\para{Correctness}
Before proving correctness, % (cf.~Claim~\ref{claim:iff} below), 
we need to state and prove two key technical facts: cogs
appearing in~\eqref{eq:delta1} are spanned by vectors from $P'$,
and so is also the vector $\Deltawp 0 \main{\vr t}$.
Our notation below relies on the implicit embedding of $\main V$ into $V = \main V \times V'$,
cf.~\eqref{eq:plus}, which allows us to consider
every vector $\vr w\in \main V$, in particular every cog, as a vector in $V$.
\begin{restatable}{claim}{claimCogLin}
%\begin{claim} 
\label{claim:coglin}
For every $\prec \; \in \Ord$, vector $\vr w \in \main {(P')}$ %such that $\supp{\vr w} \cap S = \emptyset$,
and a $k$-set $A \subseteq \A\setminus S$, 
%\slawek{Is notation $\main{(P')}$ clear?}\PH{Yes}
%\arka{I think it would be good to recall the definition here once}
\[ 
\gcog {\restr{\vr w} {\otu A k}} {\sigma^\prec_A} \ \in \ \Span {} {P'}.
\]
%\end{claim}
\end{restatable}
\begin{claimproof}
%\toremove{
%Let $T' = \supp{\vr A}$.
%W.l.o.g.~assume that  \slawek{assume that what?}
Let $\vr v\in P'$ be any vector such that $\main{\vr v} =  \vr w$.
Thus $\supp{\vr v} \cap S = \emptyset$.
For every $I\subseteq A$, we extend  $(\sigma^\prec_A)_I: A\to A\cup S$ to an 
atom automorphism $\sigma_I\in\Aut{}$ 
that acts as identity on $\supp{\vr v} \setminus A$.
%as follows:
%whenever $(\sigma_A)_I(a) = s\in S$, we put $\sigma_I(s) = a$;
%whenever $(\sigma_A)_I(a) = a$ we put $\sigma_I(s) = s$, where $s = \sigma_A(a)$;
%and $\sigma_I(x) = x$ for all $x\notin A\cup S$.
We are going to show that $\gcog {\restr{\vr w} {\otu A k}} {\sigma_A}$ is equal to the following linear combination of vectors from $P'$
(cf.~the definition~\eqref{eq:cog} of cogs):
\begin{align} \label{eq:tbs}
\gcog {\restr{\vr w} {\otu A k}} {\sigma^\prec_A} \ = \ \sum_{I\subseteq A}\, (-1)^{|I|} \cdot \sigma_I (\vr v).
\end{align}
Recalling the implicit embedding of $\main V$ and $V'$ into $V = \main V \times V'$,
we present $\vr v$ as the sum $\vr v = \vr w + \vr v'$
(recall~\eqref{eq:plus}), where  $\vr v'$ is  the projection to all non-main components.
%As the main component $\main V$ embeds as a subspace into $V$, and likewise the other components,
%we can conveniently present $\vr v$ as the sum $\vr v = \vr w + \vr v'$  of the projection $\vr w$ to the main component
%and the projection $\vr v'$ to all other components.
Furthermore, we decompose $\vr w$ into $\vr w = \restr{\vr w} {\otu A k} + \vr w'$.
Thus the right-hand side in~\eqref{eq:tbs}  decomposes into three summands:
\begin{align*}
&\sum_{I\subseteq A}\, (-1)^{|I|} \cdot \sigma_I (\vr v') \quad + \quad
\sum_{I\subseteq A}\, (-1)^{|I|} \cdot \sigma_I (\vr w') \quad + \\
&\sum_{I\subseteq A}\, (-1)^{|I|} \cdot \sigma_I (\restr{\vr w} {\otu A k}).
\end{align*}
The last one is equal to the left-hand side in~\eqref{eq:tbs} and hence it is sufficient to show that the first two summands are zero vectors.
Denote the first two summands as $\vr s_1$ and $\vr s_2$, respectively.
Recall that, given a tuple of atoms $b$ in the domain of $\vr s_1$ or $\vr s_2$, respectively, we have
\begin{align} \label{eq:s1s2}
\begin{aligned}
\vr s_1(b)\  = \ \ & \sum_{I\subseteq A}\, (-1)^{|I|} \cdot \vr v'(\sigma_I^{-1}(b)) \\
%\qquad\qquad
\vr s_2(b) \ = \ \ & \sum_{I\subseteq A}\, (-1)^{|I|} \cdot \vr w'(\sigma_I^{-1}(b)).
\end{aligned}
\end{align}
In each of the two summands, every tuple $b=(b_1, \ldots, b_{k'})$ in the domain contains less than $k$ elements of $A\cup S$:
\begin{align} \label{eq:<k}
\size{\set{b_1, \ldots, b_{k'}} \cap (A\cup S)} < k.
\end{align}
In case of $\vr s_1$ the reason is that the domain of every non-main component contains tuples $b \in \A^{k'}$ of atoms
of length $k' < k = \size A$.
In case of $\vr s_2$, while $b\in\otu \A k$, the reason is twofold: first, $\sigma_I^{-1}(b)\notin  \otu A k$
which implies $\size{\set{b_1, \ldots, b_{k'}} \cap A} < k$; second,
$S\cap \supp{\vr w} = \emptyset$ which implies $\size{\set{b_1, \ldots, b_{k'}} \cap S} = 0$.
Due to the property~\eqref{eq:<k}, 
for every tuple $b$ in the domain of a respective vector $\vr s_1$ or $\vr s_2$,
when $I$ ranges over all subsets of $A$, each tuple $\sigma_I^{-1}(b)$ appears as many times
for $I$ of odd size as for $I$ of even size. In consequence all these appearances cancel out and,
whatever the vectors $\vr v'$ and $\vr w'$ and tuple $b$ are, the right-hand sides in the two equalities~\eqref{eq:s1s2} are necessarily zero vectors.
This completes the proof of Claim~\ref{claim:coglin}.
%
%Let $a = (a_1, \ldots, a_n) \in A$ be any fixed element of $A$.
%For $I \subseteq \set{1, \ldots, n}$, let $\pi_I\in\Aut{}$ be an atom automorphism such that
%$\pi_I(a_i) = a_i$ for $i\notin I$, $\pi(a_i) = s_i$ for $i\in I$, $\pi_I(b) = b$ for all $b\in \dom{\vr v} \setminus A$,
%and $\pi_I(b)$ is arbitrary for $b\notin \dom{\vr v}$. 
%In particular, $\mixt a I s = \pi_I(a)$.
%We are going to prove the following equality (which implies the claim):
%\[
%\sum_{a \in A}  \vr v(a) \cdot \cog a s \ = \ 
%\sum_{I\subseteq \set{1,\ldots, n}} (-1)^{|I|} \cdot \pi_I(\vr v).
%\]
%
%\[
%\cog a s = \sum_{I\subseteq\set{1,\ldots, n}} \vr(a) \cdot \constvr 1 a \cdot 
%\]
%
%}
\end{claimproof}

Using Claim~\ref{claim:coglin} one further shows:
\begin{restatable}{claim}{claimDeltat}
%\begin{claim} 
\label{claim:Deltat}
$\Deltawp 0 {\main{\vr t}} \in \Span {} {P'}$.
%\end{claim}
\end{restatable}
%\arka{Maybe we should mention that the proof is in the appendix?}
%

%\para{Correctness}
%In order to complete the proof of Lemma~\ref{lem:decr-dim} we need 
It remains to show:
\begin{claim} \label{claim:iff}
$\vr t \in \Span {} {P'} \text{ if and only if } \vr {\mdf t} \in \Span {} {\mdf P}.
$
\end{claim}
\begin{claimproof}%[Proof of Claim~\ref{claim:iff}]
The `only if' direction is immediate due to Claim~\ref{claim:linmap}: if $\vr t \in \Span {} {P'}$, i.e., for some
$\ell\in\Nat$ and $q_1, \ldots, q_\ell$ and $\vr v_1, \ldots, \vr v_\ell \in P'$ we have:
\[
\vr t = q_1 \cdot \vr v_1 + \ldots + q_\ell \cdot \vr v_\ell
\]
then by Claim~\ref{claim:linmap} applied to the total order $\prec_0 \; \in \Ord$ we also have:
\begin{align} \label{eq:lincomb}
\mdf{\vr t} = q_1 \cdot (\vr v_1 - \Deltawp 0 {\main{\vr v}}_1) + \ldots + q_\ell \cdot (\vr v_\ell - \Deltawp 0 {\main{\vr v}}_\ell).
\end{align}
Therefore $\mdf{\vr t} \in \Span {}{\mdf P}$.

For `if' direction we assume $\mdf{\vr t} \in \Span {} {\mdf P}$, i.e.,
for some
$\ell\in\Nat$ and $q_1, \ldots, q_\ell$, $\vr v_1, \ldots, \vr v_\ell \in P'$, and $\prec_1, \ldots, \prec_\ell\; \in \Ord$, we have:
\[
\vr t - \Deltawp 0 {\main{\vr t}} = q_1 \cdot (\vr v_1 - \Deltawp 1 {\main{\vr v}}_1) + \ldots + 
q_\ell \cdot (\vr v_\ell - \Deltawp \ell {\main{\vr v}}_\ell).
\]
By Claim~\ref{claim:coglin} we know that $\Deltaw {\vr w} \in \Span{} {P'}$ for every $\vr w \in \main P'$ and $\prec \; \in \Ord$
(since the sum in~\eqref{eq:delta1} is essentially finite), and hence
the right-hand side is in $\Span{} {P'}$.
By Claim~\ref{claim:Deltat} we get $\vr t \in \Span{} {P'}$, as required.
\end{claimproof}

\para{Canonical form}
$(\mdf V, \mdf P, \mdf{\vr t})$ is not a formally correct instance as the space $\mdf V$~\eqref{eq:mdfV} 
is not in canonical form, and furthermore neither $\mdf V$ nor $\mdf P$ are equivariant.
$(\mdf V, \mdf P, \mdf{\vr t})$ may be however easily  transformed further into a formally correct instance as follows. 
Consider the partition of $\Aul k S$ into $S$-orbits:
\[
\Aul k S = O_1 \uplus \ldots \uplus O_{m'},
\]
for $m'\in\Nat$,
an observe that the main component $\Aul k S \magicto \Q^\ell$ of $\mdf V$ is isomorphic to
\begin{align} \label{eq:mainorbits}
(O_1 \magicto \Q^\ell) \times \ldots \times (O_{m'} \magicto \Q^\ell).
\end{align}
In case of all other components $i=2,\ldots, m$, the set $\Aul {k_i} {\neg S}$ is a single $S$-orbit.
For $i = 1, \ldots, m'$, let $r_i > 0$ denote the $S$-atom dimension of $O_i$.
As each $O_i$ is related by an $S$-supported isomorphism to $\Aul {r_1} {\neg S}$, 
the vector space~\eqref{eq:mainorbits} is related by an $S$-supported isomorphism to
\[
(\Aul {r_1} {\neg S} \magicto \Q^\ell) \ \times \ \ldots \ \times \ 
(\Aul {r_{m'}} {\neg S} \magicto \Q^\ell).
\]
We group together orbits with the same atom dimension $r_i$:
there are some pairwise different $r'_1, \ldots, r'_p\in\Nat$, and some positive  $\ell'_1, \ldots, \ell'_p\in\Nat$, such that
the main component of $\mdf V$ is 
related by an $S$-supported isomorphism to the subspace 
\begin{align} \label{eq:newcansub}
({\Aul {r'_1} {\neg S}} \magicto \Q^{\ell'_1}) \ % \times \ ({\Aul {k'_2} {\neg S}} \tofin \Q^{\ell'_2}) & \\
 \times \ \ldots \ \times \ 
({\Aul {r'_p} {\neg S}} \magicto \Q^{\ell'_p})&
\end{align}
of the vector space in canonical form
\begin{align} \label{eq:newcan}
({\otu \A {r'_1}} \tofin \Q^{\ell'_1})  % \times \ ({\otu \A {k'_2}} \tofin \Q^{\ell'_2}) & \\ 
\times  \ldots  \times  
({\otu \A {r'_p}} \tofin \Q^{\ell'_p}).
\end{align}
Relying on~\eqref{eq:mdfV} we deduce that the whole vector space $\mdf V$ is also related by an $S$-supported isomorphism to
the subspace~\eqref{eq:newcansub} of some vector space of similar form~\eqref{eq:newcan}.
Denote the latter vector space by $\mdf V'$, and observe that the subspace~\eqref{eq:newcansub} is 
exactly $\mdf V'_S$ (as defined in~\eqref{eq:VS}).

Applying the above $S$-supported isomorphism also to $\mdf P$ and $\mdf {\vr t}$,
 we get an instance
$(\mdf V', \mdf P', \mdf{\vr t}')$ equisolvable with $(V, P, \vr t)$.
Finally, we replace $\mdf P'$ by its equivariant closure
\[
\mdf P'' \ = \ \setof{\pi(\vr v)}{\pi \in \Aut{}, \ \vr v \in \mdf P'}
\]
(therefore $\mdf P' = \mdf P'' \cap \, \mdf V'_S$)
and deduce using Claim~\ref{claim:instanceS} that
the so obtained instance $(\mdf V', \mdf P'', \mdf{\vr t}')$ 
is equisolvable with $(\mdf V', \mdf P', \mdf{\vr t}')$.
%and hence also with $(V, P, \vr t)$.
%
%, where $\mdf V'$
%differs from the canonical form~\eqref{eq:canform} only by forbidding atoms from $S$.
%As the final step, by applying to $\mdf V'$ % and $\mdf P'$ 
%an arbitrary bijection $h : \A \setminus S \to \A$ that fixes $T$ (and hence preserves $\vr t$), 
%we get an equisolvable (actually isomorphic) instance 
%$(\mdf V'', \mdf P'', \mdf{\vr t})$
%whose vector space is in required canonical form
%\begin{align} \label{eq:newcan}
%({\otu \A {r'_1}} \tofin \Q^{\ell'_1})  % \times \ ({\otu \A {k'_2}} \tofin \Q^{\ell'_2}) & \\ 
%\times  \ldots  \times  
%({\otu \A {r'_p}} \tofin \Q^{\ell'_p})&
%\end{align}
%and has strictly smaller atom dimension.
%
The transformation from  $(\mdf V, \mdf P, \mdf{\vr t})$ to $(\mdf V', \mdf P'', \mdf{\vr t}')$ is effective
(cf.~Remark~\ref{rem:effect}).
Finally, each $r'_i$ in~\eqref{eq:newcan} is strictly smaller than $k$ and hence the atom dimension of $\mdf V'$ 
is smaller than that of $V$.
%\PH{I think it would be good to add example in appendix because for me it is simple but I am not sure if it will be for other people}
We have thus shown:
\begin{claim}  \label{claim:inst-comp}
$(\mdf V', \mdf P'', \mdf {\vr t}')$ in canonical form is computable from $(\mdf V, \mdf P, \mdf{\vr t})$,  
it is equisolvable with $(\mdf V, \mdf P, \mdf{\vr t})$,
and $\mdf V'$ has smaller atom dimension than $V$.
% than $(V, P, {\vr t})$.
\end{claim}

%
%Treating $h$ as a partial function from $\A$ to $\A$, we apply it to any tuple in $\otu \A k$ which results in removing
%from the tuple all elements from $S$, and replacing each atom $a\in \A\setminus S$ by $h(a)$. 
%Therefore, for every $S$-orbit $O \subseteq \otu \A k$, the image $h(O) = \otu \A \ell$, where
%$k - \ell$ is the number of elements of $S$ appearing in (all) tuples in $O$.
%As each tuple in $\mdf B$ either contains some element of $S$, or has length strictly smaller than $k$,
%applying $h$ to all $S$-orbits in $\mdf B$ yields an equivariant set 
%$h(\mdf B)$ of atom dimension strictly smaller than $k$.
%Applying further $h$ to rename the domain of the target vector $\mdf{\vr t}$, and of all vectors in $\mdf P$,
%we obtain a formally correct instance as required.

Claims~\ref{claim:comput}, \ref{claim:iff} and~\ref{claim:inst-comp}  conclude the proof of Lemma~\ref{lem:decr-dim}. % is thus completed.
\end{proof}

%% file: solv-2-fin-solv-simplified.tex
% !TEX root = main.tex

\section{Solvability reduces to finitary solvability} \label{sec:solv2fin-solv}
In this section we reduce solvability to finitary solvability:
\begin{theorem}\label{thm:gen red}
\solvname$(\Q)$ reduces to \finsolvname$(\Q)$.  %, for any unitary and effective subgroup $\X$ of $(\Q,+)$.
%for $\X \in \set{\Q, \Int}$.
\end{theorem}
%
%Recalling Remark~\ref{rem:anyF}:
%the proof works also for any other fixed field in place of $\Q$, for instance for every finite field.
%\slawek{is this recall needed?}
%\arka{I think it's fine.}

%
Let $\vr A \in \GLin{B \times C}$ and $\vr t\in\GLin B$ be the input system.
In terms of spans, the solvability problem amounts to deciding if $\vr t \in \GSpan B {\vr A}$.
We will prove the result by effectively constructing a matrix $\spread {\vr A}$ with the same row-indexing set
$B$ as $\vr A$, such that $\GSpan B {\vr A} = \Span {} {\spread {\vr A}}$.

\para{Well-definedness and exactness}
Let 
%$\vr A \in \GLin{B \times C}$ be a matrix and 
$\vr x \in \GLin C$ a vector.
We start by a characterisation of vectors $\vr x\in\GLin C$ for which the product 
$\vr y = \innerprod{\vr A}{\vr x}$ is well-defined.
Recall that $\vr y(b)$ is well-defined if and only if
there are only finitely many $c\in C$ such that 
$\vr A(b,c) \neq 0$ and $\vr x(c) \neq 0$.
Let $S = \supp {\vr A}\cup \supp{\vr x}$; in other words, $S$ is the support of the pair $(\vr A, \vr x)$.
We say that the pair $(\vr A, \vr x)$ is \emph{exact} if for every $b\in B$ and $c\in C$ 
such that $\vr A(b,c) \neq 0$ and $\vr x(c) \neq 0$ it holds
\begin{align} \label{eq:exact}
\supp c \subseteq \supp b \cup S.
%\supp c \setminus S \ \subseteq\ \supp b \setminus S.
\end{align}
%Equivalently, \eqref{eq:exact} may be written as $\supp c \subseteq \supp b \cup S$.
%
\vspace{-3mm}
\begin{lemma} \label{lem:exact}
$\innerprod {\vr A} {\vr x}$ is well-defined if and only if $(\vr A, \vr x)$ is exact. 
%\TODO{is 'exact' the right name?}\PH{I think yes, there is no obvious 
%better term and anyway this is the end of the paper so this name does not need to be remembered for a long time.}
\end{lemma}
\begin{proof}
%Let $S = \supp {\vr A}\cup \supp{\vr x}$.\arka{Already defined $S$ in the previous paragraph.}
%
For the if direction, suppose $(\vr A, \vr x)$ is exact, and consider an arbitrary fixed $b\in B$.
Let $T = \supp b \cup S$.
By~\eqref{eq:exact} the support of every $c$ satisfying $\vr A(b,c) \neq 0$ and $\vr x(c) \neq 0$
is included in $T$. 
By Claim~\ref{claim:orbitsup} in Section~\ref{sec:bases},
for every fixed set $T'\subseteq T$, % of size $k = \size {T'}$,
every orbit $O\subseteq C$ contains at most $\size{T'}!$ elements $c\in O$ such that $\supp c = T'$,
and since $C$ is orbit-finite,
there are only finitely many $c\in C$ satisfying $\vr A(b,c) \neq 0$ and $\vr x(c) \neq 0$.
%As $b$ is arbitrary, t
The product $\innerprod {\vr A} {\vr x}$ is thus well-defined, as required.

For the opposite direction, suppose $(\vr A, \vr x)$ is not exact, i.e., for some $b\in B$ and $c\in C$ we have:
\[
\vr A(b,c) \neq 0, \qquad \vr x(c) \neq 0, \qquad \supp c \not\subseteq \supp b \cup S.
\]
According to the latter condition, 
some atom $\a\in \A$ satisfies $\a\in\supp c$ and $\a\not\in T = \supp b \cup S$.
Note that every $T$-atom automorphism preserves $b$ and $\vr A$, and hence also preserves the row 
vector $\vr A(b, \_)$.
Consider an infinite family of $T$-automorphisms $\pi$ that map $\a$ to different atoms $\pi(\a) \notin T$.
For every such $\pi$ we have $\pi(c) \neq c$, but $\vr A(b, \pi(c)) = \vr A(b, c) \neq 0$.
Furthermore, every such $\pi$ preserves $\vr x$, and hence we have
$\vr x(\pi(c)) = \vr x(c) \neq 0$.
In consequence, there are infinitely many $c\in C$ such that
$\vr A(b, c) \neq 0$ and $\vr x(c) \neq 0$, i.e., the product $\innerprod {\vr A} {\vr x}$ is not well-defined on $b$.
This completes the proof.
\end{proof}

The following lemma is a crucial tool in our proof: % of Theorem~\ref{thm:gen red}:
\begin{lemma} \label{lem:wd}
Let $B$ be an orbit-finite set, $T\subseteqfin\A$, and $C$ a $T$-orbit.
Let $\vr A \in \GLin{B\times C}$
be a $T$-supported matrix and $\vr v \in \GLin C$ a vector.
If $\innerprod{\vr A}{\vr v}$ is well-defined and $\constvr 1 {O'} \in\basisof C$ appears in the basis representation of $\vr v$
then $\innerprod{\vr A}{\constvr 1 {O'}}$ is well-defined too.
\end{lemma}
\begin{proof}
%\arka{I think it will be better to keep as a lemma that every orbit-finite set is
%isomorphic to a tight orbit-finite set.}
%\slawek{isomorphic to = in finitely-supported bijection with?}
%\arka{Yes}
By Claim~\ref{claim:tightiso} in Section~\ref{sec:bases} assume w.l.o.g. that $C$ is a tight $T$-orbit.
Let $\vr v\in \GLin C$, and let $S = \supp {\vr v} \cup T$.
We follow the definition of the basis representation of $S$-supported vector $\vr v$ 
by structural induction with respect to the transitive closure of $\prec$,
as in the proof of Lemma~\ref{lem:lincomb}.
If $\odom {\vr v}$ is empty
then $\vr v$ is the zero vector and the claim holds vacuously.
Otherwise, suppose the claim holds for all strictly smaller $S$-supported vectors $\vr w$.
As in the proof of Lemma~\ref{lem:lincomb},
take an $S$-orbit $O\in\odom{\vr v}$ of maximal $S$-dimension. % (hence $O$ is forcedly infinite).
Let 
\begin{align}\label{eq:crucial}
S' := \supp c \cap S
\end{align}
for some (every) $c\in O$.
Since $T\subseteq S$ and $T\subseteq \supp c$ (as $C$ is tight), we deduce $T\subseteq S'$.
We define the $S'$-orbit $O'$ as $S'$-closure of $O$:
\[
O' \ := \ \setof{\pi(c)}{c\in O, \ \pi \in\Aut {S'}}.
\]
By definition, $S'$ is included in the support of every element of $O'$, therefore the orbit $O'$ is tight,
and hence $\constvr 1 {O'}\in \basisof C$.
According to the proof of Lemma~\ref{lem:lincomb}, the vector $\constvr 1 {O'}$
appears in the basis representation of $\vr v$, together with the vectors appearing in the basis representation of the vector
\begin{align} \label{eq:v'}
\vr w \quad := \quad \vr v \ - \ \vr v(O) \cdot \constvr 1 {O'}.
\end{align}
Note that $\vr w$ is supported by $S$ as both $\vr v$ and $\constvr 1 {O'}$ are so, and
$\vr w(O) = 0$.
By Claim~\ref{claim:ind} we infer that $\vr w \prec \vr v$
and therefore, relying on the induction assumption $\vr w$,
it is sufficient to show that $\innerprod{\vr A}{\constvr 1 {O'}}$ is well-defined.

According to the assumption and Lemma~\ref{lem:exact} we know that
$(\vr A, \vr v)$ is exact. 
Using Lemma~\ref{lem:exact} again, it is sufficient to show that $(\vr A, \constvr 1 {O'})$ is exact too.

%Recall that the domain of $\belt$ is one $S'$-orbit, namely $\refining k f$ for some $f\in \inj k$
%where $S' = \range f$.
%Let $O \subseteq C$ be the following $S$-orbit $O^{S,f}$:
%%
%\begin{align} \label{eq:O}
%O \ = \ \setof{c\in\refining k f}{\prettyforall{i\notin \dom f}{c(i) \notin S}}.
%\end{align}
%%
%Therefore $O \subseteq \dom{\vr x} \cap \dom{\belt}$.

%\slawek{the rest of the proof is according to the simplification of PH}\PH{OK}
Choose an arbitrary element $c\in O'$ and $b\in B$ such that $\vr A (b,c)\neq 0$, and an arbitrary
$S'$-atom automorphism $\pi$ such that $\pi(c)\in O$.
$\vr A$ is $T$-supported so it is also $S'$-supported (since $T\subseteq S'$).
Hence $\vr A (\pi(b), \pi(c))\neq 0$.
As $(\vr A, \vr v)$ is exact and $\vr v(O)\neq 0$, we have:
\begin{align}\label{eq:pisub}
\supp{\pi(c)}\subseteq \supp{\pi(b)}\cup S. 
\end{align}
By definition~\eqref{eq:crucial} of $S'$, as $\pi(c)\in O$, we have
$S'=\supp{\pi(c)}\cap S$, and thus the inclusion~\eqref{eq:pisub} can be strenghtened to
\[
\supp{\pi(c)}\subseteq \supp{\pi(b)}\cup S'.
\]
Application of $\pi^{-1}$ to both sides yields $\supp c \subseteq \supp b\cup S'$. 
As $b$ and $c$ were chosen arbitrarily, we conclude that
$(\vr A, \vr 1_{O'})$ is exact, as required.
%
%-------
%
%Thus $c\in \dom {\vr x}$, i.e., $\vr x(c) \neq 0$.
%%By the definition~\eqref{eq:O} of $O$ we observe:
%%\begin{align}\label{eq:crucial}
%%\supp c \cap S = S'.
%%\end{align}
%Consider any $b\in B$ such that $\vr A(b, c) \neq 0$.
%By exactness of $(\vr A, \vr x)$ we know that $\supp c \subseteq \supp b \cup S$, and thus due
%
%
%
%Choose an arbitrary element $c\in O$. 
%Thus $c\in \dom {\vr x}$, i.e., $\vr x(c) \neq 0$.
%%By the definition~\eqref{eq:O} of $O$ we observe:
%%\begin{align}\label{eq:crucial}
%%\supp c \cap S = S'.
%%\end{align}
%Consider any $b\in B$ such that $\vr A(b, c) \neq 0$.
%By exactness of $(\vr A, \vr x)$ we know that $\supp c \subseteq \supp b \cup S$, and thus due
%to~\eqref{eq:crucial} we can strengthen the latter condition, namely the column vector $\vr A(\_, c)$ satisfies:
%\begin{align} \label{eq:to}
%\supp c \subseteq \supp b \cup S' \quad \text{ for all } b\in B \text{ such that } \vr A(b,c)\neq 0.
%\end{align}
%%
%By assumption the matrix $\vr A$ is $T$-supported, hence also $S'$-supported,
%and therefore applying an arbitrary $S'$-atom automorphism $\pi$ to $\vr A(\_, c)$ yields 
%the column vector $\vr A(\_, \pi(c))$.
%As the final step we observe that the property~\eqref{eq:to} of a column vector
%is preserved by an arbitrary $S'$-atom automorphism,
%and hence holds for \emph{every} $c \in O'$.
%In other words, 
%%the condition~\eqref{eq:to} holds for every $c\in C$ for which $\belt(c) \neq 0$, which means that 
%$(\vr A, \constvr 1 {O'})$ is exact, as required.
\end{proof}

%\begin{example}
%Surprisingly, Lemma~\ref{lem:wd} does not hold if $\vr A$ is not equivariant.
%As a counterexample, consider $B = \set{\a}$ for some atom $\a \in \A$ and $C = \otu \A 2$, and a matrix 
%$\vr A \in \GLin{B\times C}$ consisting of just a single row:
%\[
%\vr A(\a, \_) \ = \ \constvr 1 {(\a, \_)}.
%\]
%Let $\vr x = \constvr 1 {} - \constvr 1 {(\a, \_)}$; thus $\dom{\vr x} = \setof{(\a', \b')\in\otu\A 2}{\a'\neq \a}$.
%The product $\innerprod{\vr A}{\vr x}$ is well-defined as the domains of $\vr A(a, \_)$ and $\vr x$ are disjoint.
%Also, according to Lemma~\ref{lem:exact}, the pair $(\vr A, \vr x)$ is exact.
%On the other hand the product $\innerprod{\vr A}{\constvr 1 {}}$ is not exact, and hence also not well-defined,
%neither is the product $\innerprod{\vr A}{\constvr 1 {(\a, \_)}}$.
%\end{example}

%
\para{Proof of Theorem~\ref{thm:gen red}}
%As we consider solutions restricted to a subgroup $\X \subseteq \Q$,
%we use the $\X$-restricted variants of definitions of $\Span {} {\vr A}$ and $\GSpan B {\vr A}$
%from Section~\ref{sec: linear eq}:
%the definition of $\Span {} {\vr A}$ as in~\eqref{eq:span2} in Section~\ref{sec:fin-solv},
%and the following variant of the definition~\eqref{eq:Gspan} of $\GSpan B {\vr A}$:
%\begin{align*} %\label{eq:Gspan}
%\GSpan {} {\vr A} = \setof{\mult{\vr A}{\vr v}}{\vr v : C\tofs\X, \ \mult{\vr A}{\vr v} \text{ well-defined}}.
%\end{align*}

Consider a system of equations $(\vr A, \vr t)$ where $\vr A\in\GLin{B\times C}$ is a matrix
and $\vr t\in\GLin B$.
Let $T = \supp{\vr A}$. % \cup \supp{\vr t}$. 
Thus $B$ and $C$ are supported by $T$ as well.
%
%As a preparatory step we claim that we can assume, w.l.o.g., that $C$ is straight.
%Indeed, if $C$ is not, we transform the matrix into $\vr A'\in\GLin{B\times C'}$.
%Indeed, supposing $C = \faktor {\otu \A k} G$ define $\vr A' : \otu \A k \to \GLin B$
%\TODO{finish}
%We thus proceed under the assumption that $C$ is equivariant and straight.

We are going to construct effectively a matrix $\spread {\vr A}$ with the same row-indexing set $B$
as $\vr A$, which
%(again, treating this matrix as a subset of $\GLin B$) 
satisfies $\GSpan B {\vr A} = \Span {} {\spread {\vr A}}$.
We claim that it is enough to consider the special case when $C$ is a single $T$-orbit.
Indeed, split the matrix $\vr A$ into $m$ matrices
\[
\vr A \ = \ \begin{bmatrix} \vr A_1 | \ldots | \vr A_m \end{bmatrix}
\]
each corresponding to one $T$-orbit $C_i \subseteq C$. 
%In the sequel we implicitly use the embeddings $\iota_i : \GLin {C_i} \to \GLin C$.
Assuming matrices $\spread{\vr A_i}$ such that
$\GSpan B {\vr A_i} = \Span {} {\spread {\vr A_i}}$ for $i = 1, \ldots, m$, we construct a matrix $\spread{\vr A}$ as
\[
\spread{\vr A} \  =  \ \begin{bmatrix} \spread{\vr A_1} | \ldots | \spread{\vr A_m} \end{bmatrix}
\]
and claim that $\GSpan B {\vr A} = \Span {} {\spread {\vr A}}$ as well. 
Indeed, $\vr v \in \GSpan B {\vr A}$ if and only if ($*$)
$\vr v = \vr v_1 + \ldots + \vr v_m$ where $\vr v_i \in \GSpan B {\vr A_i}$ for $i = 1, \ldots, m$;
replacing $\vr v_i \in \GSpan B {\vr A_i}$ by equivalent $\vr v_i \in \Span B {\spread{\vr A_i}}$
for every $i=1,\ldots, m$, the claim ($*$) is equivalent to
$\vr v \in \Span {} {\spread{\vr A}}$.
We thus proceed under the assumption that $C$ is a single $T$-orbit.
%
%Finally, we may assume, w.l.o.g., that $C$ is tight (indeed, it is enough to replace every index $c\in C$
%by the pair $(c,T)$).
%We thus proceed under the assumption that $C$ is a tight $T$-orbit.
%Note that 
Therefore $\vr A$ satisfies the assumptions of Lemma~\ref{lem:wd}. % in this case.

As the indexing set of $\spread {\vr A}$ we take those basis vectors $\belt \in \basisof C$ for which
$\innerprod {\vr A} {\belt}$
is well defined: 
\[
\spread C = \setof{\belt \in \basisof C}{\innerprod {\vr A} {\belt} \text{ is well-defined}}.
\]
The new indexing set $\spread C$ is orbit-finite as $\basisof C$ is so, and is $T$-supported
since both $\vr A$ and $\basisof C$ are  $T$-supported.
We define the new matrix $\spread{\vr A} : B \times \spread C \tofs \Q$ as follows
\[
\spread {\vr A}(\_, \belt) \ = \ \innerprod {\vr A} {\belt}.
\]
Note the injection $c\mapsto \constvr 1 c$ of $C$ into $\spread C$, as $\innerprod {\vr A}{\constvr 1 c} = \vr A(\_, c)$ is always 
well-defined.
Therefore $\spread {\vr A}$ extends $\vr A$, as $\spread{\vr A}(\_, \constvr 1 c) = 
\innerprod {\vr A}{\constvr 1 c} = \vr A(\_, c)$.
It is now sufficient to prove:
\begin{claim} \label{lem:SS}
$\GSpan {} {\vr A} = \Span {} {\spread{\vr A}}$.
\end{claim}
\begin{claimproof}
W.l.o.g.~we assume that $\vr A$ contains non-zero column vectors only
(otherwise, since $C$ is a single orbit, 
\emph{all} column vectors in $\vr A$ are zero vectors and the claim holds vacuously).
In one direction, consider any vector $\vr v \in \Span {} {\spread{\vr A}}$, i.e.,
\[
\vr v \ = \ q_1 \cdot (\innerprod{\vr A} {\belt_1}) + \ldots + q_n \cdot (\innerprod {\vr A} {\belt_n})
\]
for $q_1, \ldots, q_n\in \Q$ and $\belt_1, \ldots, \belt_n \in \spread {C}$, which immediately
yields the required membership in $\GSpan {} {\vr A}$:
\[
\vr v \ = \ \innerprod {\vr A} (q_1 \cdot \belt_1 + \ldots + q_n \cdot \belt_n) \ \in \ \GSpan {} {\vr A}.
\]
In the opposite direction, let $\vr v = \innerprod {\vr A} {\vr x}$ be well-defined for some $\vr x \in \GLin C$.
We are going to prove that $\vr v \in \Span {}{\spread{\vr A}}$.
Consider the representation of $\vr x$ in the basis $\basisof C$:
\[
\vr x \ = \ q_1 \cdot \belt_1 + \ldots + q_\ell \cdot \belt_\ell.
\]
%By the very construction of the basis representation in the proof of Lemma~\ref{lem:lincomb}, the coefficients
%$q_1, \ldots, q_\ell$ are obtained by sums and differences of coefficients $\vr x(c)$ appearing in $\vr x$ and thus belong to 
%$\X$. 
Due to Lemma~\ref{lem:wd} we know that $\innerprod {\vr A} {\belt_i}$ is well-defined
and hence $\belt_i \in \spread C$
for all $i=1,\ldots, \ell$.
Therefore
\begin{align*}
\vr v \ = \ \ & \innerprod {\vr A} {(q_1 \cdot \belt_1 + \ldots + q_\ell \cdot \belt_\ell)} \ = \ \\
& q_1 \cdot (\innerprod {\vr A}{\belt_1}) + \ldots + q_\ell \cdot (\innerprod {\vr A}{\belt_\ell}) \ = \ \\ 
& q_1 \cdot \spread{\vr A}(\_, \belt_1) + \ldots + q_\ell \cdot \spread{\vr A}(\_, \belt_\ell) \ \in \ \Span {} {\spread{\vr A}},
\end{align*}
as required.
\end{claimproof}

As discussed in Remark~\ref{rem:effect}, the transformation from $\vr A$ to $\spread{\vr A}$ is effective.
This completes the proof of Theorem~\ref{thm:gen red}.
\qed

\para{Complexity} 
%\slawek{is this correct and sufficient?}\PH{Yes}
We conclude with a rough estimation of complexity with respect to the number of orbits in $B$ and $C$,
and the atom dimension of the input system $(\vr A, \vr t)$ defined as the largest atom dimension
of each of its orbits, plus the size of its support.
%) of $\vr A$, $\vr t$, and all elements of $B\cup C$
%(in fact, %atom dimension of elements of 
%$C$ is irrelevant).

The blow-up of reduction of Theorem~\ref{thm:gen red} is exponential in the atom dimension
of input, but polynomial in the number of orbits in $B$ and $C$.
Likewise is the number and size of finite systems of equations that are produced in the procedure of 
Theorem~\ref{thm:fin-solv}.
Summing up, the combined algorithm for \solvname($\Q$) produces exponentially many 
finite systems of exponential size
(polynomially many finite systems of polynomial size, when atom dimension of input is fixed), 
and answers positively exactly when all these systems are solvable.

In the two most significant special cases, namely $\Q = \Rat$ or $\Q = \Int$, finite systems are solvable in \PTIME.
Therefore, the problems \solvname($\Rat$) and \solvname($\Int$) are in \ETIME, and likewise
are \finsolvname($\Rat$) and \finsolvname($\Int$).
When atom dimension of input is fixed, all these problems are in \PTIME.

%is exponential in $T$-atom dimension of the input system, where $T$ is it support,
%and polynomial in the number of $T$-orbits of $B$ and $C$.
%Therefore, assuming that checking if a finite system has a solution in $\X$ is in \PTIME,
%as is the case for $\Q$ and $\Int$, the problem \solvname($\X$) is in \ETIME\PH{This is wrong as for strange $\X$ the complexity can be higher}.
%\slawek{is this enough concerning complexity?}
%\arka{I think it differs slightly from what Piotrek wrote in the end of the draft}\PH{Also it is not clear if the size of the support does not go to the 
%exponent}

%% file: finalrem.tex
% !TEX root = main.tex

\section{Final remarks}

We have shown decidability of solvability of orbit-finite systems of linear equations over an arbitrary effective commutative ring.
We expect applicability of this general result in various corners; as a first example, 
combining our result with the insight of~\cite{HR21} leads to decidability of rechability
in integer-relaxation of data-enriched Petri nets. 
%\slawek{correct?}\PH{We know, not suppose}

We leave a lot of questions for further research---here we list the most important ones.
First, the immediate next step is to compute the whole solution sets represented, for instance, as a (coset of)
an orbit-finitely spanned vector subspace.
Second, an intriguing open question is whether solvability is still decidable if the finite-support restriction on solutions is dropped
(like in~\cite{DBLP:conf/lics/KlinKOT15})?
Furthermore, an important restriction on solutions is \emph{nonnegativity}, as it allows to model systems of inequalities.
According to our preliminary results \finsolvname($\Rat$), \finsolvname($\Int$) and \solvname($\Rat$) 
%\slawek{correct?}\PH{Yes}
%\arka{$\finsolvname(\Rat_{\geq 0})$ and $\finsolvname(\Int_{\geq 0})$?} 
are decidable under the nonnegativity restriction, 
%\slawek{do we know now more?}
%\arka{We know $\solvname{(\Rat)}$.}
but we don't know the status of  \solvname($\Int$).
%\arka{$\Rat_{\geq 0}$ and $\Int_{\geq 0}$?}.
Finally, in this paper we have exclusively considered equality atoms and are very curious about other richer structures.
For instance, concerning ordered atoms, the results of~\cite{HL18} indicate a huge increase of complexity of
 solvability,  compared to equality atoms.

%% file: app-missing-proofs.tex
% !TEX root = main.tex

\section{Proofs missing in Sections~\ref{sec:bases}--\ref{sec:fin-solv}}

\medskip

\claimOrbitSup*
\begin{claimproof}
\toremove{
%W.l.o.g.~assume that $O$ is an equivariant orbit.
Fix some element $x\in O$ with $\supp x = S$ and consider $\pi(x)$ for all $\pi\in\Aut{}$, 
thus ranging over all elements of $O$.
By the definition of support, if $\pi$ and $\pi'$ agree on $S$ then $\pi(x) = \pi'(x)$.
Under the condition $\supp{\pi(x)} = S$, i.e.~$\pi(S) = S$ (by Claim~\ref{claim:pisup}), 
there are only $\size{S}!$ different possibilities for $\pi$ restricted to $S$, and hence
at most that many different elements $\pi(x)$.
}
\end{claimproof}

\medskip

\claimOrbitFin*
\begin{claimproof}
\toremove{
Consider an $S$-orbit $O$ and some element $x\in O$.
If $\supp x \subseteq S$ then every $S$-atom automorphism $\pi\in\Aut S$ preserves $x$, $\pi(x) = x$, and hence
$O = \set{x}$.
Otherwise, choose any $\a \in \supp x\setminus S$ and consider, 
for each $\b\in \A\setminus (\supp x\cup S)$, some arbitrary $S$-atom automorphisms $\pi_\b$ 
that map $\a$ to $\b$ and preserves $\supp x \setminus\set{\a}$.
By Claim~\ref{claim:pisup}, $\supp{\pi_\b(x)} \neq \supp{\pi_\g(x)}$ for $\b\neq\g$, 
which implies $\pi_\b(x) \neq \pi_\g(x)$ for $\b\neq\g$.
Therefore  $O$ is infinite.
}
\end{claimproof}

\medskip

\claimTightIso*
\begin{claimproof} %{\ref{claim:tightiso}}
\toremove{
Given an $S$-orbit $O$,
the mapping $x\mapsto (x,S)$ is the required $S$-supported bijection between
$O$ and the tight $S$-orbit $\setof{(x,S)}{x\in O}$.
}
\end{claimproof}

\medskip

\claimLocSolvDec*
\begin{claimproof}
\toremove{
Consider an instance $(V, P, \vr t)$. Let $k$ be the atom dimension of $V$ and let $T = \supp P \cup \supp {\vr t}$.
The set of all $k$-sets $A\in\utu\A k$ splits into finitely many $T$-orbits (exponentially many
with respect to $k$),
and for two such $k$-sets in the same $T$-orbit the resulting restrictions are also in the same $T$-orbit.
Therefore the set of $A$-restrictions of the instance, for all $k$-sets $A$, 
splits also into finitely many $T$-orbits.
To check local solvability it is enough to checking solvability of a representative of each $T$-orbit, 
i.e., solvability of a finite number of finite systems of linear equations.
}
\end{claimproof}

%\medskip

% \claimCogLin*

\medskip

\claimDeltat*
\begin{claimproof}
%\toremove{
Let $\vr u = \main{\vr t} \in \main V$.
%\arka{Why not replace $\vr{u}$ with $\vr{t}$ everywhere?}
We use local solvability of the instance $(V, P', \vr t)$: for every $k$-set $A\subseteq \A$, 
\begin{align} \label{eq:u}
\restr{\vr u} {\otu A k} \in \Span {} {\setof{\restr {\vr w} {\otu A k}}{\vr w \in \main {(P')}}}.
\end{align}
%
%Furthermore, as $S\cap \supp{\vr t} = \emptyset$, and in consequence $S\cap\supp{\vr u} = \emptyset$, 
%we can assume that all vectors $\vr w$ involved in the finite linear
%combination above satisfy
%\begin{align} \label{eq:capsupp}
%S\cap \supp{\vr w} = \emptyset.
%\end{align}
We consider below only these finitely many subsets $A$ for which $\otu A k \cap \dom{\vr u} \neq \emptyset$.
In consequence of~\eqref{eq:u}, and because the mapping 
$\restr {\vr w}{\otu A k} \mapsto \gcog {\restr {\vr w}{\otu A k}}{\sigma^{\prec_0}_A}$ is linear, for every such $A$ we have:
\[
\gcog {\restr{\vr u} {\otu A k}} {\sigma^{\prec_0}_A} \in 
\Span {} {\setof{\gcog{\restr {\vr w} {\otu A k}}{\sigma^{\prec_0}_A}}{\vr w \in \main {(P')}}}.
\]
As $S\cap \supp{\vr t} = \emptyset$,
we have $S\cap \supp{\vr u} = \emptyset$ and hence we know that all considered subsets $A$ satisfy 
$A\subseteq \A \setminus S$.
We can thus apply Claim~\ref{claim:coglin} to
all $\vr w$ involved in the linear combination above, thus obtaining:
\[
\gcog {\restr{\vr u} {\otu A k}} {\sigma^{\prec_0}_A} \in \Span {} {P'}.
\]
Finally, the vector $\Deltawp 0 \vr u$, being a finite sum of cogs of the form $\gcog {\restr{\vr u} {\otu A k}} {\sigma^{\prec_0}_A}$, 
for finitely many subsets $A$ for which $\otu A k \cap \dom{\vr u} \neq \emptyset$,
is also in $\Span {} {P'}$, as required.
%}
\end{claimproof}

%% file: main.bbl

\begin{thebibliography}{29}

%%% the Web address.

\ifx \showCODEN    \undefined \def \showCODEN     #1{\unskip}     \fi
\ifx \showDOI      \undefined \def \showDOI       #1{#1}\fi
\ifx \showISBNx    \undefined \def \showISBNx     #1{\unskip}     \fi
\ifx \showISBNxiii \undefined \def \showISBNxiii  #1{\unskip}     \fi
\ifx \showISSN     \undefined \def \showISSN      #1{\unskip}     \fi
\ifx \showLCCN     \undefined \def \showLCCN      #1{\unskip}     \fi
\ifx \shownote     \undefined \def \shownote      #1{#1}          \fi
\ifx \showarticletitle \undefined \def \showarticletitle #1{#1}   \fi
\ifx \showURL      \undefined \def \showURL       {\relax}        \fi
% invisible to TeX
\providecommand\bibfield[2]{#2}
\providecommand\bibinfo[2]{#2}
\providecommand\natexlab[1]{#1}
\providecommand\showeprint[2][]{arXiv:#2}

\bibitem[\protect\citeauthoryear{Boja{\'n}czyk}{Boja{\'n}czyk}{2019}]%
        {atombook}
\bibfield{author}{\bibinfo{person}{Miko{\l}aj Boja{\'n}czyk}.}
  \bibinfo{year}{2019}\natexlab{}.
\newblock \bibinfo{title}{Slightly Infinite Sets}.  (\bibinfo{year}{2019}).
\newblock
\urldef\tempurl%
\url{https://www.mimuw.edu.pl/~bojan/paper/atom-book}
\showURL{%
\tempurl}


\bibitem[\protect\citeauthoryear{Boja{\'n}czyk, Klin, and Lasota}{Boja{\'n}czyk
  et~al\mbox{.}}{2014}]%
        {lmcs14}
\bibfield{author}{\bibinfo{person}{Miko{\l}aj Boja{\'n}czyk},
  \bibinfo{person}{Bartek Klin}, {and} \bibinfo{person}{Slawomir Lasota}.}
  \bibinfo{year}{2014}\natexlab{}.
\newblock \showarticletitle{Automata theory in nominal sets}.
\newblock \bibinfo{journal}{\emph{Log. Methods Comput. Sci.}}
  \bibinfo{volume}{10}, \bibinfo{number}{3} (\bibinfo{year}{2014}).
\newblock


\bibitem[\protect\citeauthoryear{Boja{\'n}czyk, Klin, and
  Moerman}{Boja{\'n}czyk et~al\mbox{.}}{2021}]%
        {BKM21}
\bibfield{author}{\bibinfo{person}{Miko{\l}aj Boja{\'n}czyk},
  \bibinfo{person}{Bartek Klin}, {and} \bibinfo{person}{Joshua Moerman}.}
  \bibinfo{year}{2021}\natexlab{}.
\newblock \showarticletitle{Orbit-Finite-Dimensional Vector Spaces and Weighted
  Register Automata}. In \bibinfo{booktitle}{\emph{Proc.~{LICS}}}.
  \bibinfo{publisher}{{IEEE}}, \bibinfo{pages}{1--13}.
\newblock


\bibitem[\protect\citeauthoryear{Boja{\'n}czyk and Toru{\'n}czyk}{Boja{\'n}czyk
  and Toru{\'n}czyk}{2012}]%
        {program-atoms}
\bibfield{author}{\bibinfo{person}{Miko{\l}aj Boja{\'n}czyk} {and}
  \bibinfo{person}{Szymon Toru{\'n}czyk}.} \bibinfo{year}{2012}\natexlab{}.
\newblock \showarticletitle{{Imperative Programming in Sets with Atoms}}. In
  \bibinfo{booktitle}{\emph{Proc.~FSTTCS 2012}}, Vol.~\bibinfo{volume}{18}.
  \bibinfo{pages}{4--15}.
\newblock


\bibitem[\protect\citeauthoryear{Clemente and Lasota}{Clemente and
  Lasota}{2015}]%
        {CL15-csl}
\bibfield{author}{\bibinfo{person}{L. Clemente} {and} \bibinfo{person}{S.
  Lasota}.} \bibinfo{year}{2015}\natexlab{}.
\newblock \showarticletitle{Reachability analysis of first-order definable
  pushdown systems}. In \bibinfo{booktitle}{\emph{Proc. {CSL}'15}}.
  \bibinfo{pages}{244--259}.
\newblock


\bibitem[\protect\citeauthoryear{Colcombet}{Colcombet}{2015}]%
        {Colcombet15}
\bibfield{author}{\bibinfo{person}{Thomas Colcombet}.}
  \bibinfo{year}{2015}\natexlab{}.
\newblock \showarticletitle{Unambiguity in Automata Theory}. In
  \bibinfo{booktitle}{\emph{Proc.~{DCFS} 2015}} \emph{(\bibinfo{series}{Lecture
  Notes in Computer Science}, Vol.~\bibinfo{volume}{9118})}.
  \bibinfo{publisher}{Springer}, \bibinfo{pages}{3--18}.
\newblock


\bibitem[\protect\citeauthoryear{Cormen, Leiserson, Rivest, and Stein}{Cormen
  et~al\mbox{.}}{2009}]%
        {Cormenbook}
\bibfield{author}{\bibinfo{person}{Thomas~H. Cormen},
  \bibinfo{person}{Charles~E. Leiserson}, \bibinfo{person}{Ronald~L. Rivest},
  {and} \bibinfo{person}{Clifford Stein}.} \bibinfo{year}{2009}\natexlab{}.
\newblock \bibinfo{booktitle}{\emph{Introduction to Algorithms, 3rd Edition}}.
\newblock \bibinfo{publisher}{{MIT} Press}.
\newblock


\bibitem[\protect\citeauthoryear{Czerwi{\'n}ski and Orlikowski}{Czerwi{\'n}ski
  and Orlikowski}{2021}]%
        {CO}
\bibfield{author}{\bibinfo{person}{Wojciech Czerwi{\'n}ski} {and}
  \bibinfo{person}{{\L}ukasz Orlikowski}.} \bibinfo{year}{2021}\natexlab{}.
\newblock \showarticletitle{Reachability in Vector Addition Systems is
  {Ackermann}-complete}. In \bibinfo{booktitle}{\emph{Proc. {FOCS} 2021}}.
  \bibinfo{publisher}{{IEEE}}, \bibinfo{pages}{1229--1240}.
\newblock
\urldef\tempurl%
\url{https://doi.org/10.1109/FOCS52979.2021.00120}
\showDOI{\tempurl}


\bibitem[\protect\citeauthoryear{Francez and Kaminski}{Francez and
  Kaminski}{1994}]%
        {KF94}
\bibfield{author}{\bibinfo{person}{Nissim Francez} {and}
  \bibinfo{person}{Michael Kaminski}.} \bibinfo{year}{1994}\natexlab{}.
\newblock \showarticletitle{Finite-Memory Automata}.
\newblock \bibinfo{journal}{\emph{Theor. Comput. Sci.}} \bibinfo{volume}{134},
  \bibinfo{number}{2} (\bibinfo{year}{1994}), \bibinfo{pages}{329--363}.
\newblock


\bibitem[\protect\citeauthoryear{Ghosh, Hofman, and Lasota}{Ghosh
  et~al\mbox{.}}{2022}]%
        {fullver}
\bibfield{author}{\bibinfo{person}{Arka Ghosh}, \bibinfo{person}{Piotr Hofman},
  {and} \bibinfo{person}{Slawomir Lasota}.} \bibinfo{year}{2022}\natexlab{}.
\newblock \showarticletitle{Solvability of orbit-finite systems of linear
  equations}.
\newblock \bibinfo{journal}{\emph{arXiv CoRR}}
  \bibinfo{volume}{abs/2201.09060} (\bibinfo{year}{2022}).
\newblock
\showeprint[arXiv]{2201.09060}
\urldef\tempurl%
\url{https://arxiv.org/abs/2201.09060}
\showURL{%
\tempurl}


\bibitem[\protect\citeauthoryear{Gupta, Shah, Akshay, and Hofman}{Gupta
  et~al\mbox{.}}{2019}]%
        {DBLP:conf/fossacs/GuptaSAH19}
\bibfield{author}{\bibinfo{person}{Utkarsh Gupta}, \bibinfo{person}{Preey
  Shah}, \bibinfo{person}{S. Akshay}, {and} \bibinfo{person}{Piotr Hofman}.}
  \bibinfo{year}{2019}\natexlab{}.
\newblock \showarticletitle{Continuous Reachability for Unordered Data {Petri}
  Nets is in {PTime}}. In \bibinfo{booktitle}{\emph{Proc.~{FOSSACS} 2019}}
  \emph{(\bibinfo{series}{Lecture Notes in Computer Science},
  Vol.~\bibinfo{volume}{11425})}. \bibinfo{publisher}{Springer},
  \bibinfo{pages}{260--276}.
\newblock


\bibitem[\protect\citeauthoryear{Hofman, Juzepczuk, Lasota, and
  Pattathurajan}{Hofman et~al\mbox{.}}{2021}]%
        {HJLP21}
\bibfield{author}{\bibinfo{person}{Piotr Hofman}, \bibinfo{person}{Marta
  Juzepczuk}, \bibinfo{person}{Slawomir Lasota}, {and} \bibinfo{person}{Mohnish
  Pattathurajan}.} \bibinfo{year}{2021}\natexlab{}.
\newblock \showarticletitle{Parikh's theorem for infinite alphabets}. In
  \bibinfo{booktitle}{\emph{Proc.~{LICS} 2021}}. \bibinfo{publisher}{{IEEE}},
  \bibinfo{pages}{1--13}.
\newblock


\bibitem[\protect\citeauthoryear{Hofman and Lasota}{Hofman and Lasota}{2018}]%
        {HL18}
\bibfield{author}{\bibinfo{person}{Piotr Hofman} {and}
  \bibinfo{person}{Slawomir Lasota}.} \bibinfo{year}{2018}\natexlab{}.
\newblock \showarticletitle{Linear Equations with Ordered Data}. In
  \bibinfo{booktitle}{\emph{Proc.~ {CONCUR} 2018}}.
  \bibinfo{pages}{24:1--24:17}.
\newblock


\bibitem[\protect\citeauthoryear{Hofman, Lasota, Lazic, Leroux, Schmitz, and
  Totzke}{Hofman et~al\mbox{.}}{2016}]%
        {Coverability-Trees-for-Petri-Nets-with-Unordered-Data}
\bibfield{author}{\bibinfo{person}{Piotr Hofman}, \bibinfo{person}{Slawomir
  Lasota}, \bibinfo{person}{Ranko Lazic}, \bibinfo{person}{J{\'{e}}r{\^{o}}me
  Leroux}, \bibinfo{person}{Sylvain Schmitz}, {and} \bibinfo{person}{Patrick
  Totzke}.} \bibinfo{year}{2016}\natexlab{}.
\newblock \showarticletitle{Coverability Trees for {P}etri Nets with Unordered
  Data}. In \bibinfo{booktitle}{\emph{Proc.~{FOSSACS} 2016}}.
  \bibinfo{pages}{445--461}.
\newblock


\bibitem[\protect\citeauthoryear{Hofman, Leroux, and Totzke}{Hofman
  et~al\mbox{.}}{2017}]%
        {DBLP:conf/lics/HofmanLT17}
\bibfield{author}{\bibinfo{person}{Piotr Hofman},
  \bibinfo{person}{J{\'{e}}r{\^{o}}me Leroux}, {and} \bibinfo{person}{Patrick
  Totzke}.} \bibinfo{year}{2017}\natexlab{}.
\newblock \showarticletitle{Linear combinations of unordered data vectors}. In
  \bibinfo{booktitle}{\emph{Proc.~{LICS} 2017}}. \bibinfo{pages}{1--11}.
\newblock


\bibitem[\protect\citeauthoryear{Hofman and R{\'{o}}zycki}{Hofman and
  R{\'{o}}zycki}{2021}]%
        {HR21}
\bibfield{author}{\bibinfo{person}{Piotr Hofman} {and} \bibinfo{person}{Jakub
  R{\'{o}}zycki}.} \bibinfo{year}{2021}\natexlab{}.
\newblock \showarticletitle{Linear equations for unordered data vectors}.
\newblock \bibinfo{journal}{\emph{arXiv CoRR}}
  \bibinfo{volume}{abs/2109.03025} (\bibinfo{year}{2021}).
\newblock
\showeprint[arXiv]{2109.03025}
\urldef\tempurl%
\url{https://arxiv.org/abs/2109.03025}
\showURL{%
\tempurl}


\bibitem[\protect\citeauthoryear{Klin, Kopczynski, Ochremiak, and
  Toru{\'n}czyk}{Klin et~al\mbox{.}}{2015}]%
        {DBLP:conf/lics/KlinKOT15}
\bibfield{author}{\bibinfo{person}{Bartek Klin}, \bibinfo{person}{Eryk
  Kopczynski}, \bibinfo{person}{Joanna Ochremiak}, {and}
  \bibinfo{person}{Szymon Toru{\'n}czyk}.} \bibinfo{year}{2015}\natexlab{}.
\newblock \showarticletitle{Locally Finite Constraint Satisfaction Problems}.
  In \bibinfo{booktitle}{\emph{Proc.~{LICS} 2015}}. \bibinfo{pages}{475--486}.
\newblock


\bibitem[\protect\citeauthoryear{Kosaraju}{Kosaraju}{1982}]%
        {Kosaraju82}
\bibfield{author}{\bibinfo{person}{S.~Rao Kosaraju}.}
  \bibinfo{year}{1982}\natexlab{}.
\newblock \showarticletitle{Decidability of Reachability in Vector Addition
  Systems (Preliminary Version)}. In \bibinfo{booktitle}{\emph{Proc.~{STOC}
  1982}}. \bibinfo{pages}{267--281}.
\newblock


\bibitem[\protect\citeauthoryear{Lambert}{Lambert}{1992}]%
        {Lambert92}
\bibfield{author}{\bibinfo{person}{Jean{-}Luc Lambert}.}
  \bibinfo{year}{1992}\natexlab{}.
\newblock \showarticletitle{A Structure to Decide Reachability in {P}etri
  Nets}.
\newblock \bibinfo{journal}{\emph{Theor. Comput. Sci.}} \bibinfo{volume}{99},
  \bibinfo{number}{1} (\bibinfo{year}{1992}), \bibinfo{pages}{79--104}.
\newblock


\bibitem[\protect\citeauthoryear{Lasota}{Lasota}{2016}]%
        {Lasota16}
\bibfield{author}{\bibinfo{person}{Slawomir Lasota}.}
  \bibinfo{year}{2016}\natexlab{}.
\newblock \showarticletitle{Decidability Border for {Petri} Nets with Data:
  {WQO} Dichotomy Conjecture}. In \bibinfo{booktitle}{\emph{Proc.~{PETRI}
  {NETS} 2016}} \emph{(\bibinfo{series}{Lecture Notes in Computer Science},
  Vol.~\bibinfo{volume}{9698})}. \bibinfo{publisher}{Springer},
  \bibinfo{pages}{20--36}.
\newblock


\bibitem[\protect\citeauthoryear{Lasota}{Lasota}{2022}]%
        {Lasota22}
\bibfield{author}{\bibinfo{person}{S{\l}awomir Lasota}.}
  \bibinfo{year}{2022}\natexlab{}.
\newblock \showarticletitle{Improved {Ackermannian} lower bound for the {Petri}
  nets reachability problem}. In \bibinfo{booktitle}{\emph{Proc.~STACS 2022}}
  \emph{(\bibinfo{series}{LIPIcs}, Vol.~\bibinfo{volume}{219})}.
  \bibinfo{publisher}{Schloss Dagstuhl - Leibniz-Zentrum f{\"{u}}r Informatik},
  \bibinfo{pages}{46:1--46:15}.
\newblock
\urldef\tempurl%
\url{https://doi.org/10.4230/LIPIcs.STACS.2022.46}
\showDOI{\tempurl}


\bibitem[\protect\citeauthoryear{Lazic, Newcomb, Ouaknine, Roscoe, and
  Worrell}{Lazic et~al\mbox{.}}{2008}]%
        {Nets-with-Tokens-which-Carry-Data}
\bibfield{author}{\bibinfo{person}{Ranko Lazic},
  \bibinfo{person}{Thomas~Christopher Newcomb}, \bibinfo{person}{Jo{\"{e}}l
  Ouaknine}, \bibinfo{person}{A.~W. Roscoe}, {and} \bibinfo{person}{James
  Worrell}.} \bibinfo{year}{2008}\natexlab{}.
\newblock \showarticletitle{Nets with Tokens which Carry Data}.
\newblock \bibinfo{journal}{\emph{Fundam. Inform.}} \bibinfo{volume}{88},
  \bibinfo{number}{3} (\bibinfo{year}{2008}), \bibinfo{pages}{251--274}.
\newblock


\bibitem[\protect\citeauthoryear{Lazic and Totzke}{Lazic and Totzke}{2017}]%
        {What-Makes-Petri-Nets-Harder-to-Verify-Stack-or-Data}
\bibfield{author}{\bibinfo{person}{Ranko Lazic} {and} \bibinfo{person}{Patrick
  Totzke}.} \bibinfo{year}{2017}\natexlab{}.
\newblock \showarticletitle{What Makes {P}etri Nets Harder to Verify: Stack or
  Data?}. In \bibinfo{booktitle}{\emph{Concurrency, Security, and Puzzles -
  Essays Dedicated to Andrew William Roscoe on the Occasion of His 60th
  Birthday}}. \bibinfo{pages}{144--161}.
\newblock


\bibitem[\protect\citeauthoryear{Leroux}{Leroux}{2021}]%
        {L}
\bibfield{author}{\bibinfo{person}{J{\'{e}}r{\^{o}}me Leroux}.}
  \bibinfo{year}{2021}\natexlab{}.
\newblock \showarticletitle{The Reachability Problem for {Petri} Nets is Not
  Primitive Recursive}. In \bibinfo{booktitle}{\emph{Proc. {FOCS} 2021}}.
  \bibinfo{publisher}{{IEEE}}, \bibinfo{pages}{1241--1252}.
\newblock
\urldef\tempurl%
\url{https://doi.org/10.1109/FOCS52979.2021.00121}
\showDOI{\tempurl}
\newblock
\shownote{To appear}.


\bibitem[\protect\citeauthoryear{Leroux and Schmitz}{Leroux and
  Schmitz}{2015}]%
        {LerouxS15}
\bibfield{author}{\bibinfo{person}{J{\'{e}}r{\^{o}}me Leroux} {and}
  \bibinfo{person}{Sylvain Schmitz}.} \bibinfo{year}{2015}\natexlab{}.
\newblock \showarticletitle{Demystifying Reachability in Vector Addition
  Systems}. In \bibinfo{booktitle}{\emph{Proc.~{LICS} 2015}}.
  \bibinfo{pages}{56--67}.
\newblock


\bibitem[\protect\citeauthoryear{Mayr}{Mayr}{1981}]%
        {Mayr81}
\bibfield{author}{\bibinfo{person}{Ernst~W. Mayr}.}
  \bibinfo{year}{1981}\natexlab{}.
\newblock \showarticletitle{An Algorithm for the General {P}etri Net
  Reachability Problem}. In \bibinfo{booktitle}{\emph{Proc.~{STOC} 1981}}.
  \bibinfo{pages}{238--246}.
\newblock


\bibitem[\protect\citeauthoryear{Neven, Schwentick, and Vianu}{Neven
  et~al\mbox{.}}{2004}]%
        {NSV04}
\bibfield{author}{\bibinfo{person}{Frank Neven}, \bibinfo{person}{Thomas
  Schwentick}, {and} \bibinfo{person}{Victor Vianu}.}
  \bibinfo{year}{2004}\natexlab{}.
\newblock \showarticletitle{Finite state machines for strings over infinite
  alphabets}.
\newblock \bibinfo{journal}{\emph{{ACM} Trans. Comput. Log.}}
  \bibinfo{volume}{5}, \bibinfo{number}{3} (\bibinfo{year}{2004}),
  \bibinfo{pages}{403--435}.
\newblock


\bibitem[\protect\citeauthoryear{Pitts}{Pitts}{2013}]%
        {Pitts:book}
\bibfield{author}{\bibinfo{person}{A.~M. Pitts}.}
  \bibinfo{year}{2013}\natexlab{}.
\newblock \bibinfo{booktitle}{\emph{Nominal Sets: Names and Symmetry in
  Computer Science}}. \bibinfo{series}{Cambridge Tracts in Theoretical Computer
  Science}, Vol.~\bibinfo{volume}{57}.
\newblock \bibinfo{publisher}{Cambridge University Press}.
\newblock
\showISBNx{9781107017788}


\bibitem[\protect\citeauthoryear{Su{\'{a}}rez, Teruel, and Colom}{Su{\'{a}}rez
  et~al\mbox{.}}{1996}]%
        {SilvaTC96}
\bibfield{author}{\bibinfo{person}{Manuel~Silva Su{\'{a}}rez},
  \bibinfo{person}{Enrique Teruel}, {and} \bibinfo{person}{Jos{\'{e}}~Manuel
  Colom}.} \bibinfo{year}{1996}\natexlab{}.
\newblock \showarticletitle{Linear Algebraic and Linear Programming Techniques
  for the Analysis of Place or Transition Net Systems}. In
  \bibinfo{booktitle}{\emph{Lectures on Petri Nets {I:} Basic Models, Advances
  in Petri Nets}}. \bibinfo{pages}{309--373}.
\newblock


\end{thebibliography}
